\documentclass{article}

\usepackage{microtype}
\usepackage{graphicx}
\usepackage{subfigure}
\usepackage{booktabs} 
\usepackage{amsthm}
\usepackage{enumitem}

\usepackage{hyperref}
\usepackage{amsmath}

\newcommand{\myparagraph}[1]{\textbf{#1}}

\usepackage[accepted]{icml2021}

\icmltitlerunning{Link Prediction with Persistent Homology: An Interactive View}

\newtheorem{theorem}{Theorem}

\newtheorem{lemma}[theorem]{Lemma}
\usepackage{etoolbox}

\begin{document}

\twocolumn[
\icmltitle{Link Prediction with Persistent Homology: An Interactive View}




\begin{icmlauthorlist}
\icmlauthor{Zuoyu Yan}{pku}
\icmlauthor{Tengfei Ma}{ibm}
\icmlauthor{Liangcai Gao}{pku}
\icmlauthor{Zhi Tang}{pku}
\icmlauthor{Chao Chen}{sbu}
\end{icmlauthorlist}

\icmlaffiliation{pku}{Wangxuan Institute of Computer Technology, Peking University, Beijing, China}
\icmlaffiliation{sbu}{Department of Biomedical Informatics, Stony Brook University, New York, USA}
\icmlaffiliation{ibm}{T. J. Watson Research Center, IBM, New York, USA}

\icmlcorrespondingauthor{Chao Chen}{chao.chen.1@stonybrook.edu}
\icmlcorrespondingauthor{Liangcai Gao}{glc@pku.edu.cn}

\icmlkeywords{Persistence homology, Graph Neural Network, Link Prediction}

\vskip 0.3in
]



\printAffiliationsAndNotice{}  

\begin{abstract}
Link prediction is an important learning task for graph-structured data. 
In this paper, we propose a novel topological approach to characterize interactions between two nodes. Our topological feature, based on the extended persistent homology, encodes rich structural information regarding the multi-hop paths connecting nodes. Based on this feature, we propose a graph neural network method that outperforms state-of-the-arts on different benchmarks. 
As another contribution, we propose a novel algorithm to more efficiently compute the extended persistence diagrams for graphs. This algorithm can be generally applied to accelerate many other topological methods for graph learning tasks.
\end{abstract}

\section{Introduction}

Graph-structured data is very common in our life. Learning from graphs is important in various scientific and industrial domains  \cite{zhang2020deep,wu2020comprehensive}. In this paper, we focus on the link prediction task, i.e., to learn to predict whether an edge exists between two target nodes, conditioned on their attributes and local connectivity \cite{liben2007link,schlichtkrull2018modeling,lu2011link}. Link prediction is an important step in knowledge discovery in various applications, e.g., recommendation systems \cite{koren2009matrix,adamic2003friends}, knowledge graph completion \cite{teru2020inductive}, protein-protein interactions \cite{coulomb2005gene}, and gene prediction \cite{nagarajan2015predicting}.

Classic link prediction methods \cite{barabasi1999emergence,zhou2009predicting,brin2012reprint} use hand-crafted connectivity features and enforce strong assumptions of the distributions of links, e.g., nodes with similar connectivity features tend to be connected.
Better performance has been achieved by comparing the similarity of nodes in the embedding space \cite{perozzi2014deepwalk}, which encodes more global connectivity information. 
In recent years, graph neural networks (GNNs) have achieved state-of-the-art link prediction performance as they exploit graph connectivity information and node attributes in a completely data driven manner.
However, even for GNNs, graph connectivity information such as node degrees is beneficial;

it provides contextual information for the graph convolutional operations \cite{kipf2016semi,qiu2018network,ye2020curvature}.

For link prediction, an effective strategy is the direct modeling of the interaction between two target nodes. The path distance of nearby nodes from the target nodes has been shown useful \cite{zhang2018link}. 
However, these distance-based methods mainly focus on the ``closeness'' between target nodes, but do not explicitly model the ``richness of connections'' between them. 
In Figure \ref{fig:teaser}, we show examples with the same distance encoding. But the connections in the right example are much richer than in the left example. 
It is conceivable that nodes with a wealth of multi-hop connections have a better chance to share an edge. 

\begin{figure}[h!]
	\centering
	\includegraphics[width=.7\columnwidth]{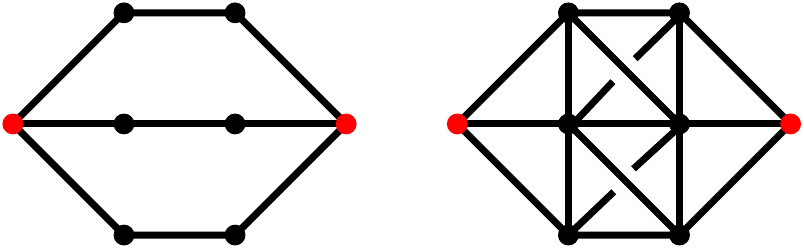}
	\vspace{-.1in}
	\caption{Two example graphs. In both cases, black nodes have the same distance from the red target nodes (either (2,1) or (1,2)). But the richness of the connections is very different. On the left, there are only three connections between the target nodes. On the right, there are many more connections between the targets. }
	\label{fig:teaser}
	\vspace{-.1in}
\end{figure}

\begin{figure*}[btp]
	\centering
	\subfigure[]{
		\begin{minipage}[t]{0.22\linewidth}
			\centering
			\includegraphics[width=\columnwidth]{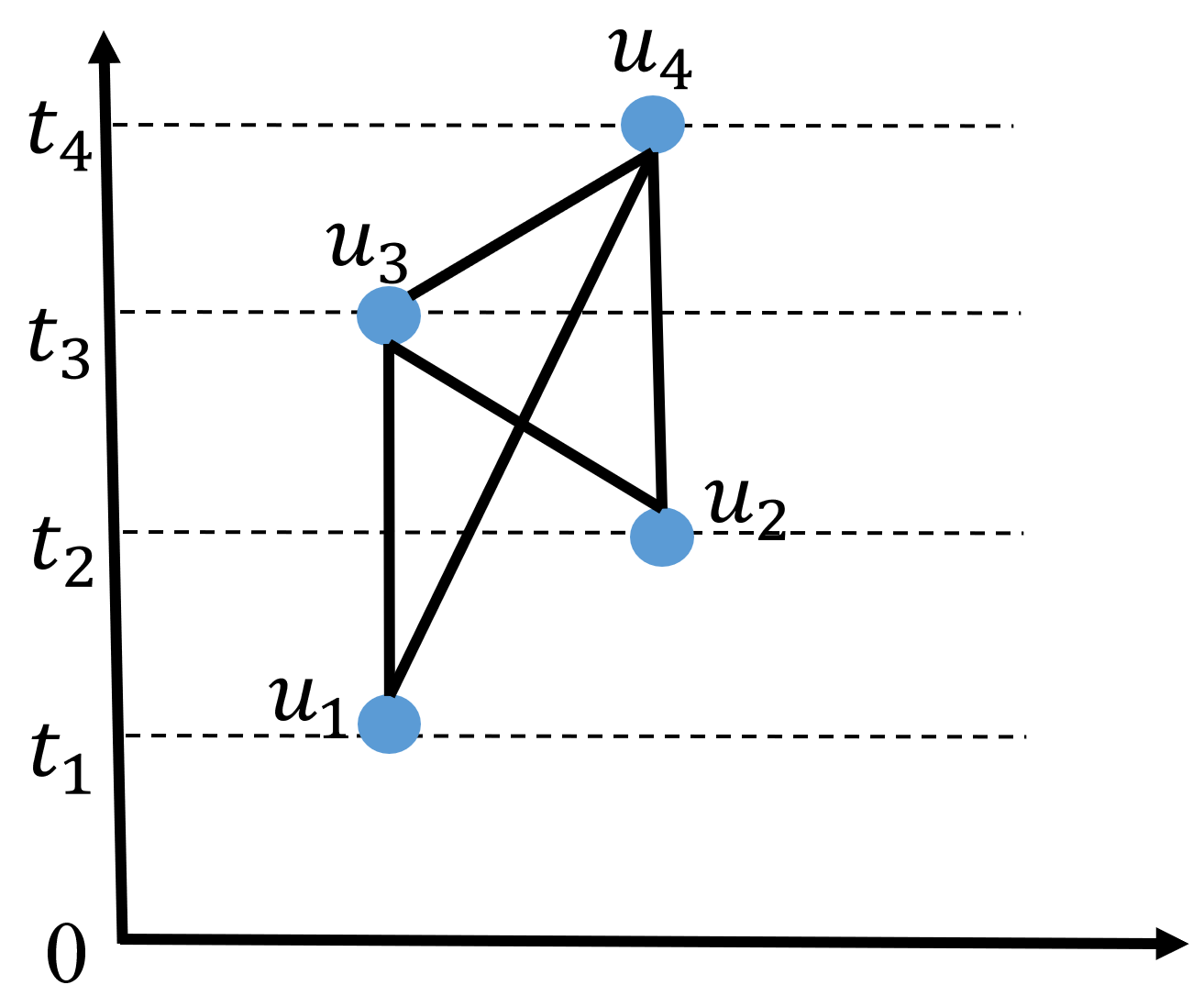}
		\end{minipage}
	}%
	\subfigure[]{
		\begin{minipage}[t]{0.46\linewidth}
			\centering
			\includegraphics[width=\columnwidth]{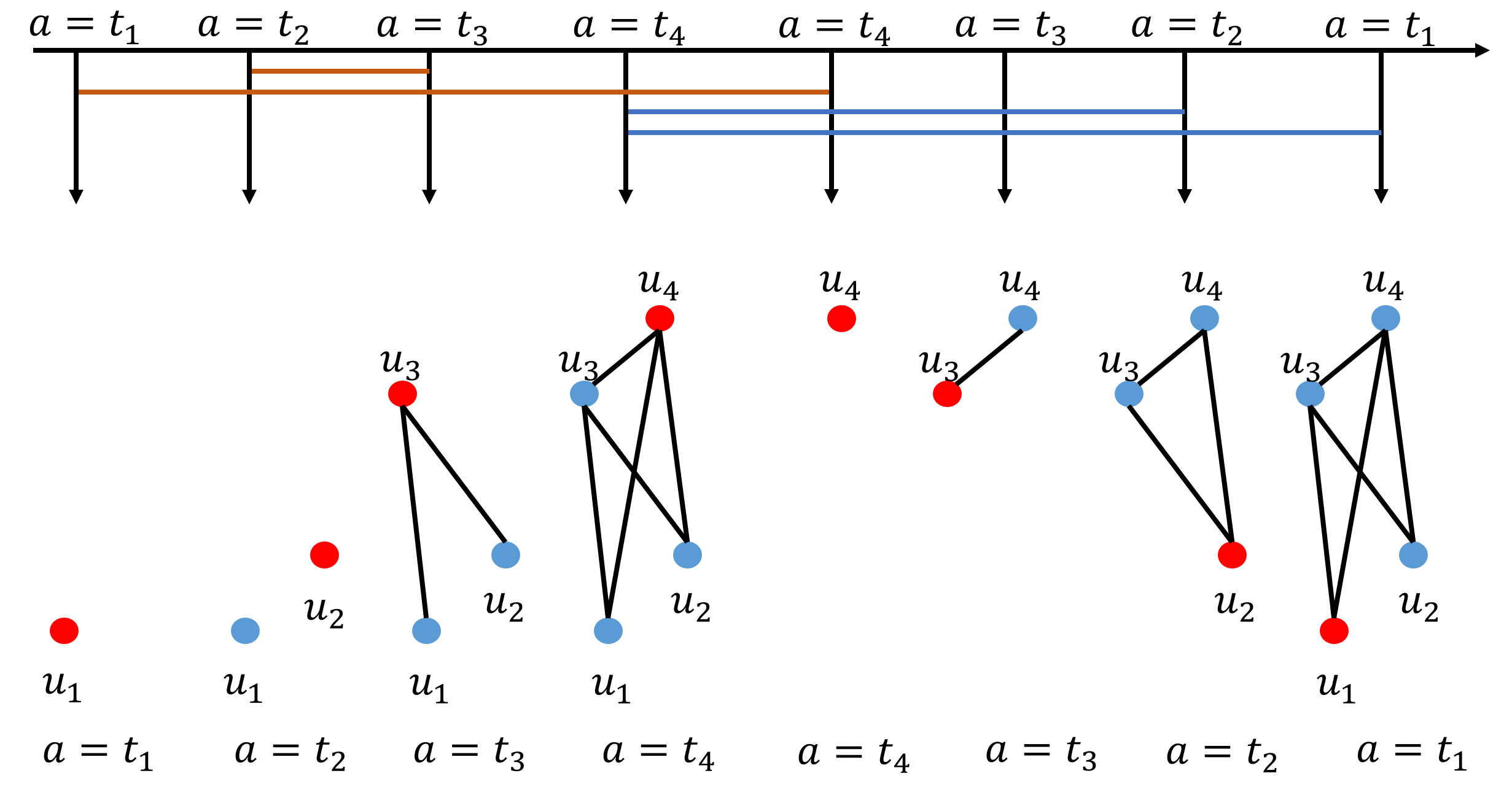}
		\end{minipage}%
	}%
	\subfigure[]{
		\begin{minipage}[t]{0.26\linewidth}
			\centering
			\includegraphics[width=\columnwidth]{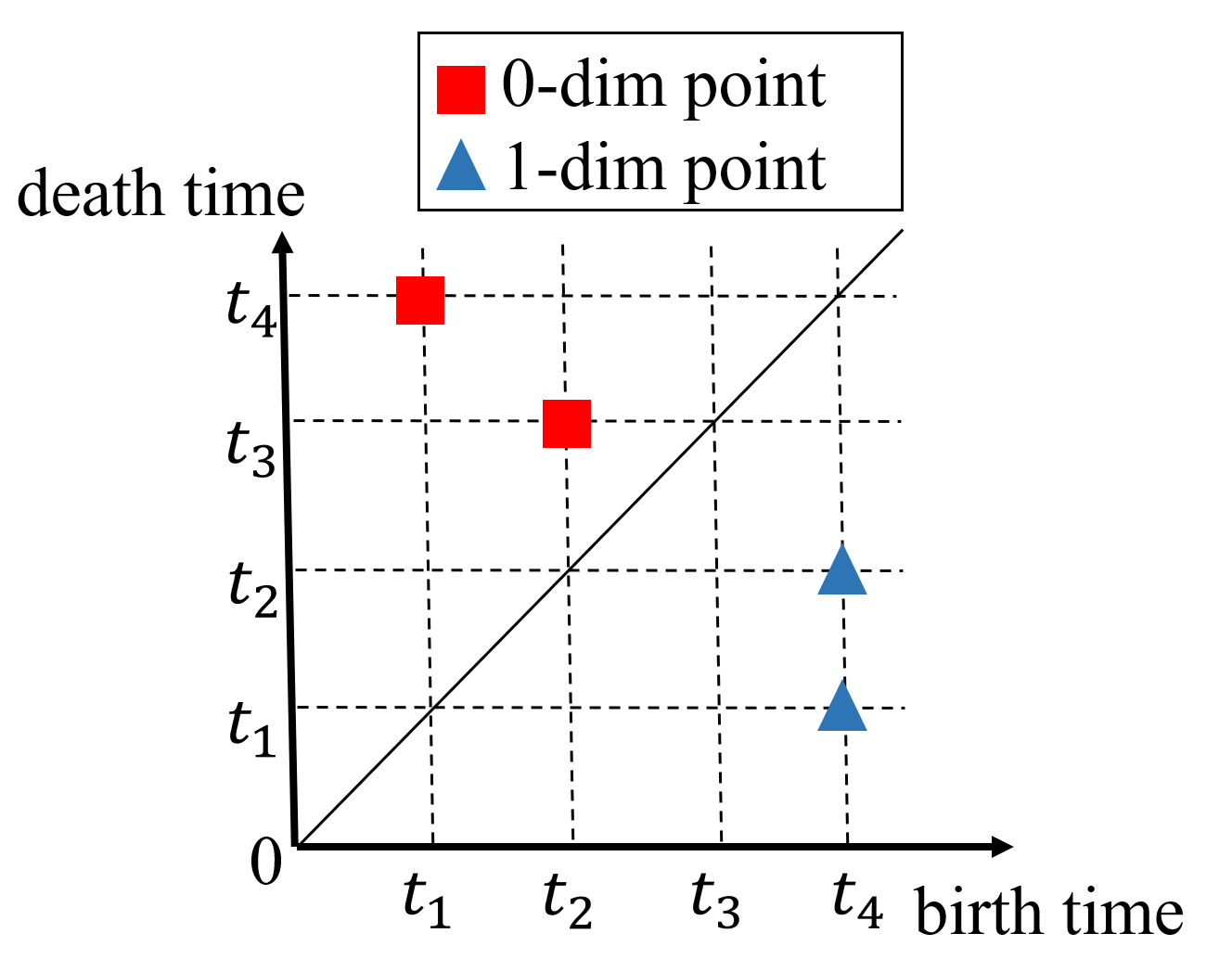}
		\end{minipage}
	}%
	\centering
	\vspace{-.1in}
	\caption{An illustration of extended persistent homology. (a) We plot the input graph with a given filter function. The filter value for each node is $f(u_1)=t_1$, $f(u_2)=t_2$, $f(u_3)=t_3$, $f(u_4)=t_4$. (b) The ascending and descending filtrations of the input graph. The bars of brown and blue colors correspond to the life spans of connected components and loops respectively. The first four figures are the ascending filtration, while the last four figures denote the descending filtration. In the ascending filtration, $f(uv) = max(f(u),f(v))$, while in the descending filtration, $f(uv) = min(f(u),f(v))$. (c) In the resulting extended persistence diagram, red and blue markers correspond to 0-dimensional and 1-dimensional topological structures. There are two blue markers, corresponding to two loops $(u_1u_3,u_3u_4,u_4u_1)$, $(u_2u_3,u_3u_4,u_4u_2)$. The range of filter function $f$ for these two loops are $[t_1,t_4]$, $[t_2,t_4]$ respectively. These ranges are encoded as the coordinates of the blue markers.}
	\label{fig:motivation}
	\vspace{-.15in}
\end{figure*}

To exploit the richness of connections between nodes, we propose a novel method based on the theory of persistent homology \cite{edelsbrunner2000topological,edelsbrunner2010computational}, which encodes high-order structural information of the graph via algebraic topology. The theory captures topological structures of arbitrary shape and scale, e.g., connected components and loops, and encodes them in a robust-to-noise manner.
To predict whether two given nodes $u$ and $v$ are connected, our method explicitly counts the number of loops within their vicinity. 
This count essentially measures the complexity of connections between $u$ and $v$, which can be indicative of the existence of links.
The method not only counts the numbers of loops, but also measures the range of distance from $u$ and $v$ for each loop. See Figure \ref{fig:motivation} for an illustration. This rich set of structural information is mapped into a topological feature space and is integrated into a graph neural network. As will be shown in experiments, our topological loop-counting graph neural network achieves better performance for link prediction.

Persistent homology has been used for learning with graphs \cite{zhao2019learning,hofer2020graph,hofer2017deep,carriere2020perslay}. However, most existing works use it as a global structural feature for the whole graph. These global features, although proven useful for graph classification tasks, cannot describe the interaction between a pair of nodes.

In this paper, we propose a \emph{pairwise topological feature} to capture the richness of the interaction between a specific pair of target nodes.
We compute topological information within the vicinity of the target nodes, i.e., the intersection of the $k$-hop neighborhoods of the nodes. It has been shown that such \emph{local enclosing graph} carries sufficient information for link prediction \cite{zhang2018link}.  

To measure the saliency of topological structures, we also introduce a distance-based filter function to measure the interaction between nodes.
These choices ensure our pairwise topological feature to be informative of the interaction between the targets. 

We propose \emph{topological loop-counting graph neural network (TLC-GNN)} by injecting the pairwise topological feature into the latent representation of a graph neural network. Our method achieves state-of-the-art performance in link prediction.
Another contribution of this paper is on the computational side.
To capture the full information of loops, we use the extended persistent homology \cite{cohen2009extending}.

The commonly used algorithm is based on a matrix reduction algorithm that is similar to the Gaussian elimination. It is cubic to the input graph/subgraph size, $O((|V|+|E|)^3)$.\footnote{In theory, the fastest algorithm for persistence diagram has the same complexity as matrix multiplication \cite{milosavljevic2011zigzag}, i.e., $O((|V|+|E|)^\omega)$, in which $\omega=2.3728596$ \cite{alman2021refined}.} 
Since the computation needs to be executed on all training/validation/testing pairs of nodes, we could significantly benefit from a faster algorithm.
In this paper, we propose a novel algorithm for the extended persistent homology, specific for graphs. Instead of the expensive matrix reduction, our algorithm directly operates on the graph and is quadratic to the input graph/subgraph size, $O(|V||E|)$. This algorithm is not application specific and can be applied to other graph learning problems \cite{zhao2019learning,hofer2020graph,hofer2017deep,carriere2020perslay}.

In summary, our contribution is three-fold:
\begin{itemize}[topsep=0pt, partopsep=0pt,itemsep=0pt,parsep=0pt]
	\item We introduce a pairwise topological feature based on persistent homology to measure the complexity of interaction between nodes. We compute the topological feature specific to the target nodes using a carefully designed filter function and domain of computation. 
	\item We use the pairwise topological feature to enhance the latent representation of a graph neural network and achieve state-of-the-art link prediction results on various benchmarks.
	\item We propose a general-purpose fast algorithm to compute extended persistent homology on graphs. The time complexity is improved from cubic to quadratic to the input size. It applies to many other persistent-homology-based learning methods for graphs.
\end{itemize}

\myparagraph{Outline.} In Section~\ref{sec:related}, we briefly introduce existing works on link prediction and on learning with topological information. In Section~\ref{sec:LP with PH}, we present details of extended persistent homology and our model, TLC-GNN. In Section~\ref{sec:acc}, we introduce a faster algorithm for extended persistent homology and prove its correctness. In Section~\ref{sec:exp}, we evaluate our method on synthetic and real-world benchmarks.

\section{Related Work}
\label{sec:related}

\myparagraph{Link prediction methods.}
Early works \cite{barabasi1999emergence,zhou2009predicting,brin2012reprint,jeh2002simrank} predict links based on node similarity scores measured within the local neighborhood of the two target nodes. 

These methods tend to have strong assumptions on the link distribution and do not generalize well. 
Graph embeddings have been used to encode more global structural information  for link prediction \cite{koren2009matrix,airoldi2008mixed,perozzi2014deepwalk,tang2015line,qiu2018network}. 
 
However, these methods only rely on graph connectivity and do not take full advantage of node attributes. 
 
In recent years, new GNN-based methods have been proposed to jointly leverage graph connectivity and node attributes.
\citet{zhang2018link} show that local enclosing subgraphs contains sufficient information, and propose a GNN to leverage such information. 
Considering the non-Euclidean nature of graph metrics, one may generalize graph convolution to the hyperbolic space \cite{chami2019hyperbolic,zhu2020graph}. However, most existing methods use either limited structural information or node embedding to represent edge features. They do not explicitly model the advanced topological information that arises in node interaction.

\myparagraph{Learning with topological features.}
Persistent homology \cite{edelsbrunner2000topological,edelsbrunner2010computational} captures structural information from the data using the language of algebraic topology \cite{munkres2018elements}. It captures multi-scale topological structures in a provably robust manner \cite{cohen2007stability}. 

Different learning methods for persistent homology have been proposed, such as direct vectorization \cite{adams2017persistence}, kernel machines \cite{reininghaus2015stable,kusano2016persistence,carriere2017sliced}, convolutional neural networks \cite{hofer2017deep}, topological loss \cite{chen2019topological,hu2019topology,hofer2019connectivity}, and generative model \cite{wang2020topogan}.

For graph-structured data, topological features have been used for node classification \cite{zhao2020persistence} and graph classification \cite{zhao2019learning,hofer2020graph,carriere2020perslay}. 
However, these existing methods cannot model interactions between nodes as desired in link prediction tasks.
\citet{bhatia2018understanding} also use persistent homology for link prediction. But their method only exploits 0-dimensional topology, i.e., whether the target nodes are connected or not within the local neighborhood. This cannot capture the complexity of connection as we intend to model. 

\section{Link Prediction with Persistence Homology}
\label{sec:LP with PH}

In this section, we introduce our topological loop-counting graph neural network (TLC-GNN), which computes persistent homology based on the chosen subgraph and incorporates the pairwise topological feature into a graph neural network. The input of the model includes two target nodes and a subgraph encoding their topological information. The output is the probability of whether an edge exists between the two target nodes.

In section~\ref{sec:extended} and ~\ref{sec:MR}, we will briefly introduce the extended persistent homology and its computation, respectively. In section~\ref{sec:tlcgnn}, we will illustrate how to combine the topological feature with a standard graph neural network.

\subsection{Extended Persistent Homology}
\label{sec:extended}
In this section, we provide a brief introduction to extended persistent homology and refer the reader to \cite{cohen2009extending} for details. In the setting of graphs, the data only contain 0-dimensional (connected components) and 1-dimensional (loops) topological structures\footnote{In graphs, there is no triangle. As a result, all the loops are topological structures (non-bounding cycles). }. We define simplices to be all elements in the graph, including nodes and edges. The combinatorial relationship between simplices determines the topological structure of the graph, and persistent homology can count the number of these topological structures. Besides, persistent homology measures the saliency of all topological structures in view of a scalar function defined on all simplices, called the filter function.
For example, let $V$, $E$ be the sets of nodes and edges, and $X$ be the union of $V$ and $E$, namely the set of simplices. The filter function for nodes $f: V\rightarrow \mathbf{R}$ can be defined as the sum of the distance to the target nodes. Then we can further define the filter function for edge $uv$ as the maximum value of $f(u)$ and $f(v)$. 

Given the filter function for all the simplices in a graph, we can define $X_a$ as the sublevel set of $X$: $X_a = \{x|f(x) \leq a, x\in X\}$. Here $a$ is a threshold in the filter function $f$, and $X_a$ is the subset of $X$ whose filter value are not greater than $a$. As the threshold $a$ increases from $-\infty$ to $\infty$, we obtain the ascending filtration of $X$: $\emptyset = X_{-\infty} \subset ...\subset X_{\infty} = X$. An example is shown in the first half of Figure~\ref{fig:motivation} (b). 

With the increasing of threshold $a$, the sublevel set grows from empty to $X$, and new topological structures gradually appear (born) and disappear (die). For example, two connected components
appear when reaching $X_{t_1}$ and $X_{t_2}$ (for simplicity, we replace $X_{t_i}$ with $X_i$). One of the connected components 
disappears in $X_3$. Besides, two loops 
appear when reaching $X_4$.

\myparagraph{Extended persistence.} However, with the settings above, we find that some structures, such as the whole connected component and all the loops, will never disappear. 
To address this limitation of ordinary persistent homology, \citet{cohen2009extending} propose extended persistence by introducing another filtration of the superlevel set $X^a = \{x|f(x)\geq a, x \in X\}$. Let $a$ decrease from $\infty$ to $-\infty$, and we can obtain the descending filtration of $X$: $\emptyset = X^{\infty} \subset ...\subset X^{-\infty} = X$\footnote{Technically, the superlevel set should be the relative homology groups $(X,X^{a})$ in the second half of the filtration.}. An example descending filtration is shown in the second half of Figure \ref{fig:motivation} (b). In the descending filtration, the filter function for an edge $uv$ is set differently, i.e., $f(uv)$ = $min(f(u),f(v))$.

For loops and the whole connected component in a graph, the death time can be defined as the filter value in the superlevel set when the structure appears again. For instance, the extended persistence point of the loop $\{ac,cd,da\}$ in Figure \ref{fig:motivation} is $(t_4, t_1)$. It is born when reaching $X_4$ in the ascending filtration and dies when reaching $X^1$ in the descending filtration. It is a bit counter-intuitive that the death time can be smaller than the birth time.

After capturing the birth, death times of all the topological structures, we encode them into a 2-D point set called persistence diagram. Each topological structure corresponds to one persistence point in the diagram. Its x and y coordinates are the birth and death times. In Figure~\ref{fig:motivation} (c), we show an extended persistence diagram.

After obtaining the extended persistence diagram, we can encode it into a vectorized feature called persistence image \cite{adams2017persistence}. 

Further details are available in the supplementary material.
\subsection{Matrix Reduction Algorithm for Extended Persistence Diagram}
\label{sec:MR}

In this section, we will introduce the algorithm to compute the extended persistence diagrams. Let $m$ be the total number of simplices ($m=|V|+|E|$). We write ($\kappa_1,\kappa_2,...,\kappa_m$) as the ascending sequence of simplices in $X$, i.e., $f(\kappa_1) < f(\kappa_2) < ... < f(\kappa_m)$\footnote{Without loss of generality, we assume the filter function of all nodes are distinct.}. Similarly, we write ($\lambda_1,\lambda_2,...,\lambda_m$) as the descending sequence of simplices in $X$. Every simplex will appear once in the ascending sequence and once in the descending sequence. 

To compute the extended persistence diagram, we need a binary valued matrix $M$ to encode the adjacency relationship between nodes and edges. Matrix $M$ is a $2m\times 2m$ matrix consisting of four $m\times m$ matrices: $M=\left[\begin{matrix} A & P \\ 0 & D\end{matrix} \right] $. Every column or row of $M$ corresponds to a simplex. In particular, the first $m$ columns of $M$ correspond to the ascending sequence of simplices $\kappa_1,...,\kappa_m$. The last $m$ columns of $M$ correspond to the descending sequence of simplices $\lambda_1,...,\lambda_m$. The setting is the same for the rows of $M$. Matrix $A$ encodes the relationship between all the simplices in the ascending sequence. Similar to the incidence matrix of a graph, 
$A[i,j]=1$ iff $\kappa_i$ is the boundary of $\kappa_j$, i.e., $\kappa_i$ is a node adjacent to the edge $\kappa_j$. The matrix $A$ of Figure~\ref{fig:motivation} is shown in Figure~\ref{fig:A}
\begin{figure}[ht]
	\begin{center}
		\centerline{\includegraphics[width=\columnwidth]{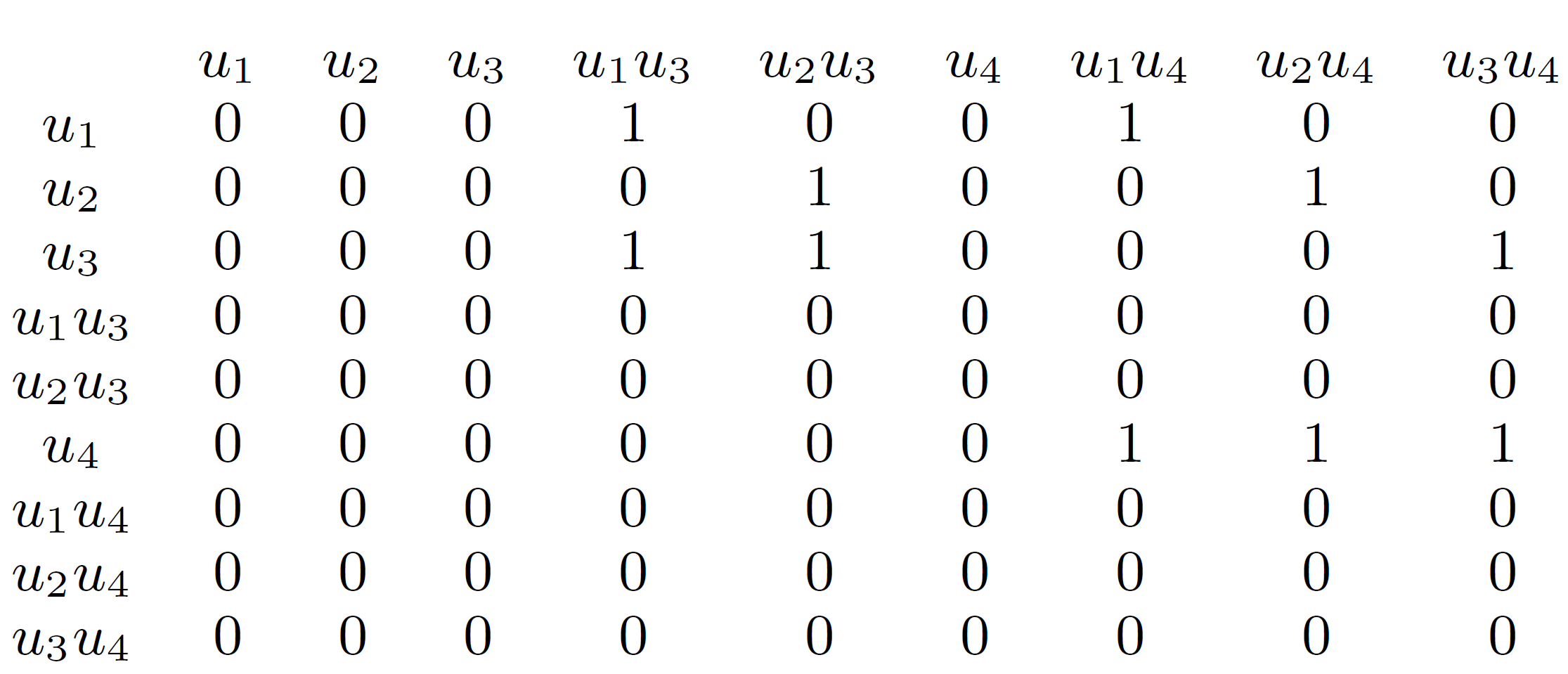}}
		\caption{The matrix $A$ of Figure~\ref{fig:motivation}}
		\label{fig:A}
	\end{center}
\end{figure}

$D$ is defined similarly, except that it encodes the relationship of the simplices in the descending sequence, i.e., $D[i,j] = 1$ iff $\lambda_i$ is the boundary of $\lambda_j$. $P$ stores the permutation that connects the two sequences of simplices, i.e., $P[i,j]=1$ iff $\kappa_i$ and $\lambda_j$ denote the same simplex. $0$ is a zero-valued $m\times m$ matrix. In the supplementary material, we will provide a complete example of matrix $M$.

The computation of persistent homology boils down to counting ranks of submatrices of the boundary matrix. This can be achieved by a specific matrix reduction algorithm on $M$. Algorithm~\ref{alg:MR} reduces $M$ from left to right. To describe the reduction process, let $low_M(j)$ be the maximum row index $i$ for which $M[i,j] = 1$. For instance, in Figure~\ref{fig:A}, $low_M(4) = 3$.
If column $j$ is zero, then $low_M(j)$ is undefined. $M$ is reduced if $low_M(j) \neq low_M(k)$ for any two non-zero columns $ j\neq k$. Notice that ``add" here is the mod-2 sum of the binary column vectors. After the matrix reduction is finished, we can record the persistence diagram. Each pair of simplex ids $(low_M(j), j)$ corresponds to an extended persistence pair. By taking the filter function value of the corresponding simplices, we will obtain the corresponding birth and death times of a persistence point. For simplicity, we slightly abuse the notation, that is to denote the birth and death time as $f(low_M(j))$ and $f(j)$ respectively. To better illustrate Algorithm~\ref{alg:MR}, we provide the reduction process of Figure~\ref{fig:motivation} in the supplementary material.

\begin{algorithm}[h]
	\caption{Matrix Reduction}
	\label{alg:MR}
	\begin{algorithmic}
		\STATE {\bfseries Input:} filter funtion $f$, graph $G$
		\STATE Persistence Diagram $PD=\{\}$
		\STATE$M=$ build reduction matrix$(f,G)$
		\FOR{$j=1$ {\bfseries to} $2m$}
		\WHILE{$\exists k < j$ with $low_M(k)=low_M(j)$}
		\STATE add column $k$ to column $j$
		\ENDWHILE
		\STATE add $(f(low_M(j)),f(j))$ to $PD$
		\ENDFOR
		\STATE {\bfseries Output:} Persistence Diagram $PD$
	\end{algorithmic}
\end{algorithm}

\subsection{TLC-GNN for Link Prediction}
\label{sec:tlcgnn}
We have introduced extended persistent homology, its computation and vectorization. Next, we explain how our method computes topological feature for a specific pair of target nodes, and combine it with GNN for link prediction.

To characterize the interaction between two target nodes, we need to choose a vicinity graph and a corresponding filter function. Given a weighted graph, we can define the sum of the distance from any node $v$ to the target nodes $v_1$ and $v_2$ as the filter function, $f(v)=d(v,v_1)+d(v,v_2)$. Note that for each pair of target nodes on which we predict whether a link exists, we create a filter function and then compute the persistent homology as the topological feature. 

Our topological feature is only extracted within vicinity of the target nodes.
Existing works using persistent homology usually exploit topological information from the whole graph \cite{zhao2019learning,hofer2020graph}, thus cannot be directly used in our task. Besides, topological structures captured by global persistent homology can be irrelevant to the target nodes if they are far away. As justified in \cite{zhang2018link}, local enclosing subgraphs already contain enough information for link prediction. We exploit the intersection of the $k$-hop neighborhoods of the two target nodes: let $G = (V,E)$ be the whole graph, where $V$ and $E$ are the set of vertices and edges. $V_1 = \{v|d(v,v_1) \leq k\}$ and $V_2 = \{v|d(v,v_2) \leq k\}$ are the k-hop neighborhoods of target nodes $v_1$ and $v_2$, $V_{12} = V_1 \cap V_2$ is the intersection. The enclosing subgraph is $G_{12} = (V_{12},E_{12})$,  where  $E_{12} = E\cap V_{12}^2$. We compute the persistence image of the subgraphs generated by all possible links, and define $PI(v_1,v_2)$ as the persistence image of the subgraph generated by target nodes $v_1$ and $v_2$.

We combine the extracted topological feature $PI(v_1,v_2)$ with GNN-generated node embedding features for link prediction.
To obtain node embeddings, we use a classic $L$-layer graph convolutional network (GCN) \cite{kipf2016semi}. Messages are passed between nodes to update their feature representations. After an $L$-layer GCN, the embedding of a certain node can be viewed as a combination of node representation from its $L$-hop neighborhood. $H^{\ell} = [h_1^{\ell}, h_2^{\ell},...,h_{|V|}^{\ell}]$ where $h^{\ell}_i \in \mathbf{R}^{d_{\ell}}$ is the representation of node $i$ in the $\ell$-th layer, $\ell=0,1,...,L$, and $|V|$ is the number of nodes. Here, $H^0$ is the input node features, and $H^{L}$ is the node embeddings of the final layer.

To compute the representation of all the nodes, in the $\ell$-th layer, $H^{\ell} = \sigma(\hat D^{-1/2}\hat A  \hat D^{-1/2} H^{\ell-1}  W^{\ell})$, where $\hat A = A + I$ denotes the adjacency matrix with inserted self-loop and $\hat D$ is a diagonal matrix with $\hat D[i,i] = \sum_j \hat A[i,j]$. $\sigma$ represents the activation function, and $W^{\ell}$ is a learned matrix to encode node embedding from dimension $d_{\ell-1}$ to $d_\ell$. The nodewise formulation of node $i$ is given by $h_i^{\ell} = \sigma(W^{\ell}\sum_{j \in N(i)} \frac{1}{\sqrt{\hat d_j \hat d_i}}h_j^{\ell-1})$ where $N(i)$ is the neighborhood of node $i$ including $i$ itself.

For given target nodes $u$ and $v$, the node embeddings $h_u$ and $h_v$ obtained through the GCN model can be viewed as the node feature, and the persistence image $PI(u,v)$ can be viewed as the edge feature. To combine the two features effectively, we use a modified Fermi-Dirac decoder \cite{krioukov2010hyperbolic, nickel2017poincare}: $(h_u-h_v)^2$ is defined as the distinction between the two nodes. It is then concatenated with $PI(u,v)$, and passed to a two layer Multi-Layer Perceptron (MLP), which outputs a single value. Denote the value after the two layer MLP as distance of the two target nodes: $dist(u,v)$, and the final probability of whether there exists an edge between $u$ and $v$ is $prob(u,v) = \frac{1}{(e^{(dist(u,v)-2)} + 1)}$. We then train TLC-GNN by minimizing the cross-entropy loss using negative sampling.

\section{A Faster Algorithm for Extended Persistence Diagram}
\label{sec:acc}
In this section, we propose a new algorithm for extended persistent homology.
Recall that Algorithm~\ref{alg:MR} reduces the matrix $M$. In the new algorithm, we manage to avoid explicit construction and reduction of the matrix $M$. Instead, we create and maintain a rooted tree while going through the descending filtration. When processing a new edge, by inspecting its relationship with the current tree, we can find the corresponding persistence pair efficiently. Afterward, we update the tree and continue with the next edge in the descending filtration. In section \ref{sec:accal}, we explain the algorithm in details.
In section \ref{sec:proof}, we prove the correctness of the proposed algorithm.

\newcommand{\T}{\mathcal{T}}

\subsection{A Faster Algorithm}
\label{sec:accal}
\begin{algorithm}[tb]
	\caption{A Faster Algorithm for Extended Persistence Diagram}
	\label{alg:Acc}
	\begin{algorithmic}
		\STATE {\bfseries Input:} filter funtion $f$ of descending filtration, graph $G$, filter function $f_a$ of ascending filtration
		\STATE 0-dim PD, $E_{pos}$, $E_{neg}$ = Union-Find($G,f$)
		\STATE Tree $\T$ = $E_{neg}$ + all the nodes
		\STATE 1-dim PD = $\{\}$
		\FOR{$e_j$ in $E_{pos}$}
		\STATE assume $e_j=uv$, $Path_u\subseteq \T$ is the path from $u$ to $r$ within the tree $\T$, $Path_v \subseteq \T$ is the path from $v$ to $r$
		\STATE $Loop = Path_u \cup Path_v -  Path_u \cap Path_v$ + $e_j$
		\STATE add $(max_{e\in Loop} f_a(e),f(e_j))$ to 1-dim PD
		\STATE $ e_k = argmax_{e\in Loop} f_a(e)$
		\STATE $\T = \T - \{e_k\} + \{e_j\}$
		\ENDFOR
		\STATE {\bfseries Output:} 0-dim PD, 1-dim PD 
	\end{algorithmic}
\end{algorithm}

\begin{figure}[ht]
	\begin{center}
		\centerline{\includegraphics[width=\columnwidth]{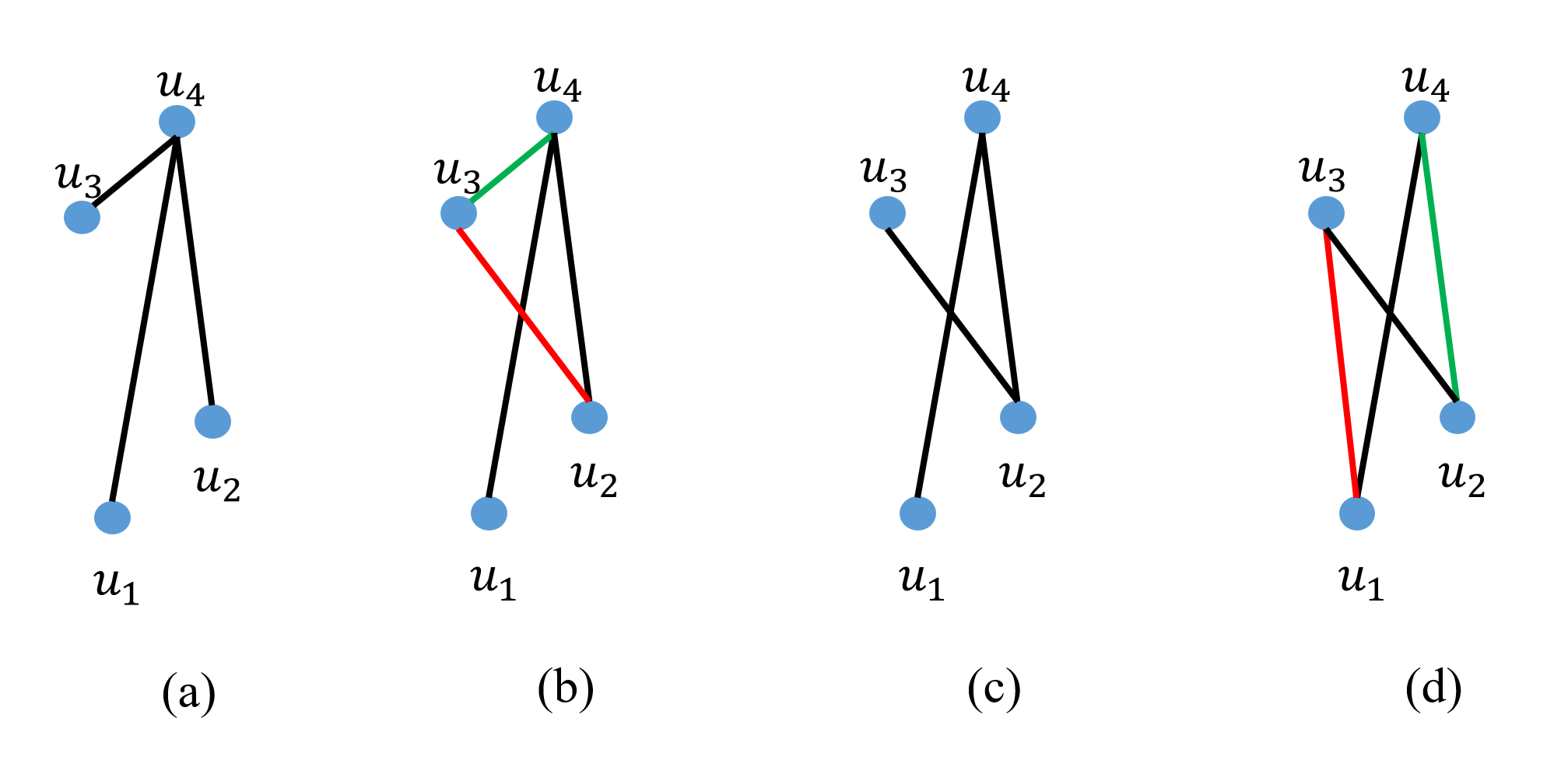}}
		\caption{An illustration of Algorithm~\ref{alg:Acc}, based on Figure~\ref{fig:motivation}.}
		\label{fig:acc_example}
	\end{center}
\end{figure}

Our new algorithm is shown in Algorithm~\ref{alg:Acc}. For 0-dim topology, we run existing union-find algorithm \cite{edelsbrunner2010computational} for the ascending filtration once and for the descending filtration once. The algorithm is guaranteed to be correct. In the following section, we will therefore mainly focus on 1-dimensional extended persistence.

Take Figure \ref{fig:acc_example} as an example, with the Union-Find algorithm \cite{edelsbrunner2000topological,cohen2006vines}, we can procure the set of positive edges and the set of negative edges in the descending filtration: $E_{pos} = \{u_2u_3, u_1u_3\}$, $E_{neg} = \{u_3u_4,u_2u_4,u_1u_4\}$. Recall that positive edges give rise to new loops, while negative edges do not. And in the descending filtration, every positive edge will be paired with a certain edge in the ascending filtration, and give rise to a certain loop. 

In order to find all the persistence pairs, we construct a tree with all the negative edges and nodes. Every time we add a positive edge into the tree, a loop appears. The newly added edge and the edge which appears the latest in the ascending filtration in the loop form the persistence pair. We then delete the paired edge to update the tree iteratively. Updating the tree is essential. When a new edge $e$ is added into the descending filtration, it forms a unique loop with the tree; this loop is exactly the loop with the latest birth during ascending and is coned out due to $e$. Without tree updating, the loop formed by the tree and $e$ will not be the desired one.

Look at Figure~\ref{fig:acc_example}, all the negative edges and nodes form (a). The positive edge $u_2u_3$ is then added, and we can discover the loop it forms: $\{u_3u_4,u_4u_2,u_2u_3\}$. In the loop, the edge that appears the latest in the ascending filtration is $u_3u_4$. The sequence of simplices in the ascending filtration is shown in Figure~\ref{fig:A}. Then $u_2u_3$ is paired with $u_3u_4$, we can obtain the persistence point $(t_4,t_2)$. 

After that, we delete the paired edge $u_3u_4$ to get (c). Then the positive edge $u_1u_3$ is added, with the same process, we can discover the loop: $\{u_1u_3,u_3u_2,u_2u_4,u_4u_1\}$, and the paired edge is $u_2u_4$. The persistence point of the loop is $(t_4,t_1)$, so the final 1-dimensional extended persistence diagram is $\{(t_4,t_2),(t_4,t_1)\}$.


\myparagraph{Complexity of Algorithm~\ref{alg:Acc}.} For every positive edge, the time cost to get the loop and to form the new path are both $O(|V|)$, here $|V|$ denotes the number of nodes, and $|E|$ represents the number of edges. So the time cost among all positive edges is $(|E|-|V|+1)*O(|V|) = O(|V||E|)$. 
The union-find algorithm will take $O(|V|+|E|)log(|E|+|V|)$. 
Therefore the final computational cost is $O(|V||E|)$. This is much more efficient than the complexity of the matrix reduction algorithm, $O((|V|+|E|)^3)$.

\subsection{The Correctness of Algorithm \ref{alg:Acc}}
\label{sec:proof}

\newcommand{\redM}{\overline{M}}
\newcommand{\redD}{\overline{D}}
\newcommand{\redP}{\overline{P}}
\newcommand{\redA}{\overline{A}}

In this section, we prove the proposed faster algorithm is correct, i.e., it produces the same output as the matrix reduction algorithm. Formally, we state the theorem as follows.
\begin{theorem}
	\label{theorem}
	Algorithm~\ref{alg:Acc} outputs the same extended persistence diagram as Algorithm~\ref{alg:MR}.
\end{theorem}
We briefly explain the idea of the proof. A complete proof is provided in the supplementary material. 

The 0-dimensional persistent homology is computed using known Union-Find algorithm. We only need to prove the correctness for the 1-dimensional topology, i.e., loops. Recall each loop is created by an edge in the ascending filtration and destroyed by an edge in the descending filtration. We plan to show that the pairs of the edges found by our new algorithm are the same as the pairs found by the matrix reduction algorithm. Denote by $\redM=\left[\begin{matrix} \redA & \redP \\ 0 & \redD\end{matrix} \right] $ the reduced matrix of $M$ and its submatrices. 

We classify edges in the descending filtration into negative descending edges and positive descending edges. A negative descending edge $e_i$ destroys a connected component created by a descending node $v_j$. In the reduced matrix $\redM$, its lowest entry $[v_j,e_i]$ falls into $\redD$. A positive descending edge $e_i$ adds a new loop to the superlevel set. But for extended persistence, it destroys the loop created by an edge $e_j$ in the ascending filtration. The lowest entry of $e_i$, $[e_j, e_i]$ falls into $\redP$. See Figure \ref{fig:motivation} (c) for an illustration of different types of simplex pairs and their corresponding  
persistence points in the extended diagram.

The core of our proof is to show that for each positive descending edge, its corresponding pair found by the new algorithm is equivalent to the pair in the reduced algorithm. We prove this by induction. As we go through all positive descending edges in the descending filtration, we show that for edge $e_i$, the pairing by our algorithm is the same as the reduction result. We prove that:

\begin{itemize}[topsep=0pt, partopsep=0pt,itemsep=0pt,parsep=0pt]
	\item After reducing the $D$ part of the column of $e_i$, i.e., when $low(e_i)\leq m$, the remaining entries of $e_i$ in $P$ constitute a unique loop in $\{e_i\}\cup \T_{i-1}$. Here $\T_i$ is the tree after updating the first $i$ positive edges.
	\item The lowest entry of the reduced column, $e_j$, is not paired by any other previous edges, and thus will be paired with $e_i$. Indeed, it is the last edge in the loop w.r.t.~the ascending ordering.
	\item The updating of the tree $\T_i = \T_{i-1} - \{e_j\} + \{e_i\}$ is equivalent to adding the reduced column of $e_i$ to all descending columns with nonzero entry at row of $e_j$. Although this will affect the final reduced matrix, it will not change the resulting simplex pairings.
\end{itemize}

This proves that our new algorithm has the same output as the existing matrix reduction algorithm. We have shown that the new algorithm is much more efficient in terms of complexity (Section~\ref{sec:accal}). We will also validate the benefit empirically in Section~\ref{sec:exp}.

\section{Experiments}
\label{sec:exp}
\begin{table*}[t]
	\caption{Mean and standard deviation of ROC-AUC  on real-world data. ``*'': results copied from \cite{chami2019hyperbolic,zhu2020graph}.} \label{tab:result}
	\vskip 0.15in
	\begin{center}
		\begin{small}
			\begin{sc}
				\begin{tabular}{lllll}
					\hline\noalign{\smallskip}
					Method &  PubMed & Photo & Computers \\
					\noalign{\smallskip}\hline\noalign{\smallskip}
					GCN \cite{kipf2016semi} & 89.56$\pm$3.660* & 91.82$\pm$0.000 & 87.75$\pm$0.000\\
					HGCN \cite{chami2019hyperbolic}  & 96.30$\pm$0.000* & 95.40$\pm$0.000 & 93.61$\pm$0.000  \\
					GIL \cite{zhu2020graph}  &  95.49$\pm$0.160* & 97.11$\pm$0.007 & 95.89$\pm$0.010  \\
					SEAL \cite{zhang2018link} & 92.42$\pm$0.119 & 97.83$\pm$0.013 & 96.75$\pm$0.015 \\
					\noalign{\smallskip}
					\hline
					\noalign{\smallskip}
					PEGN \cite{zhao2020persistence} & 95.82$\pm$0.001  & 96.89$\pm$0.001 & 95.99$\pm$0.001\\
					TLC-GNN (nodewise) & 96.91$\pm$0.002 & 97.91$\pm$0.001 & 97.03$\pm$0.001\\
					\noalign{\smallskip}
					\hline
					\noalign{\smallskip}
					TLC-GNN (DRNL) & 96.89$\pm$0.002  & 97.61$\pm$0.003 & 97.23$\pm$0.003  \\
					TLC-GNN (Ricci) & \textbf{97.03$\pm$0.001} & \textbf{98.23$\pm$0.001} & \textbf{97.90$\pm$0.001}\\
					\noalign{\smallskip}\hline
				\end{tabular}
			\end{sc}
		\end{small}
	\end{center}
	\vskip -0.1in
\end{table*}

We evaluate our method on on synthetic and real-world graph datasets. We compare with different SOTA link prediction baselines. Furthermore, we evaluate the efficiency of the proposed faster algorithm for extended persistent homology. Source code will be available at \href{https://github.com/pkuyzy/TLC-GNN}{https://github.com/pkuyzy/TLC-GNN}.

\myparagraph{Baseline methods.}
We compare our framework TLC-GNN with different link prediction methods. We compare with popular GNN models such as \textbf{GCN} \cite{kipf2016semi} and \textbf{GAT} \cite{velivckovic2017graph}.
We use them for node embedding and then adopt the Fermi-Dirac decoder \cite{krioukov2010hyperbolic, nickel2017poincare} to predict whether there is a link between two nodes. 
We also compare with several SOTA methods. \textbf{SEAL} \cite{zhang2018link}, which utilize local distances and pre-trained transductive node features for link
prediction.
\textbf{HGCN} \cite{chami2019hyperbolic}, which introduces hyperbolic space to deal with free-scale graphs. 
\textbf{GIL} \cite{zhu2020graph}, which takes advantage of both Euclidean and hyperbolic geometries.

To demonstrate the importance of \emph{pairwise} topological features, we also compare with two baseline methods using node-wise topological features, i.e., topological feature for each node.  \textbf{PEGN} \cite{zhao2020persistence} extracts topological features for each node using its neighborhood. Then the node-wise topological feature is used to re-calibrate the graph convolution. PEGN was originally designed for node classification. Similar to GCN and GAT, we adapt it to link prediction via the Fermi-Dirac decoder. A second baseline, \textbf{TLC-GNN (Nodewise)}, is a modification of our method TLC-GNN (Ricci). Instead of the pairwise topological feature for a pair of target nodes, we concatenate their node-wise topological features, inject it into the node embedding, and then use MLP to predict. 

For all settings, we randomly split edges into 85/5/10\% for training, validation, and test sets. To compute the distance-based filter function, we need a graph metric. We use hop-distance and Ollivier-Ricci curvature \cite{ni2018network} respectively. The filter function with hop-distance is the same as Double-Radius Node Labeling (DRNL) \cite{zhang2018link}. For Ollivier-Ricci curvature, we compute the curvature for all training edges and use them as the edge weights. We add 1 to the curvature of all edges to avoid negative edge weights. With this weighted graph, we define the filter function of each node as the total shortest distance from the two target nodes. In each subgraph, the corresponding target nodes $(v_1, v_2)$ are used as the anchors to compute the filter function: for given node $v$, $f(v) = d(v, v_1) + d(v, v_2)$, where $d(v, v_1)$ is the shortest path distance between the given node $v$ and the corresponding target node $v_1$ in the weighted graph.
Recall that we need to use the intersection of the $k$-hop neighborhoods of the two target nodes to compute pairwise topological features. The choice of $k$ is either 1 or 2, depending on the size and density of the graph.

Note that here we use pre-computed weight function to define filter function. It is possible to extend the framework to learn the filter function end-to-end, as in graph classification \cite{hofer2020graph}. However, this involves recomputing filter function and persistence diagram for each pair of target nodes for each epoch; it is computationally prohibitive. Alternatively, we may learn better kernels on the persistence diagrams, which do not require recomputing persistence diagrams every epoch \cite{zhao2019learning}.

\subsection{Synthetic Experiments}

We generate synthetic data using a graph theoretical model called Stochastic Block Model (SBM) \cite{holland1983stochastic}. 
To be specific, we create random graphs with 1000 nodes, forming 5 equal-size communities. Edges are randomly sampled with intra-community probability $p$ and inter-community probability $q$. We randomly create 12 graphs with $p$ ranging in $\{0.05,0.25,0.45\}$  and $q$ ranging in $\{0.0, 0.015, 0.03, 0.045\}$. Furthermore, we assign each node with a randomly created feature of dimension 100 and use them as the input of all the methods. For pairwise persistence diagrams, we use intersections of 1-hop neighborhoods for all the graphs.

We run the methods on each synthetic graph 10 times and report the mean average area under the ROC curve (ROC-AUC) score as the result. More details can be found in the supplementary material. 

\begin{figure}[btp!]
	\begin{center}
		\centerline{\includegraphics[width=\columnwidth]{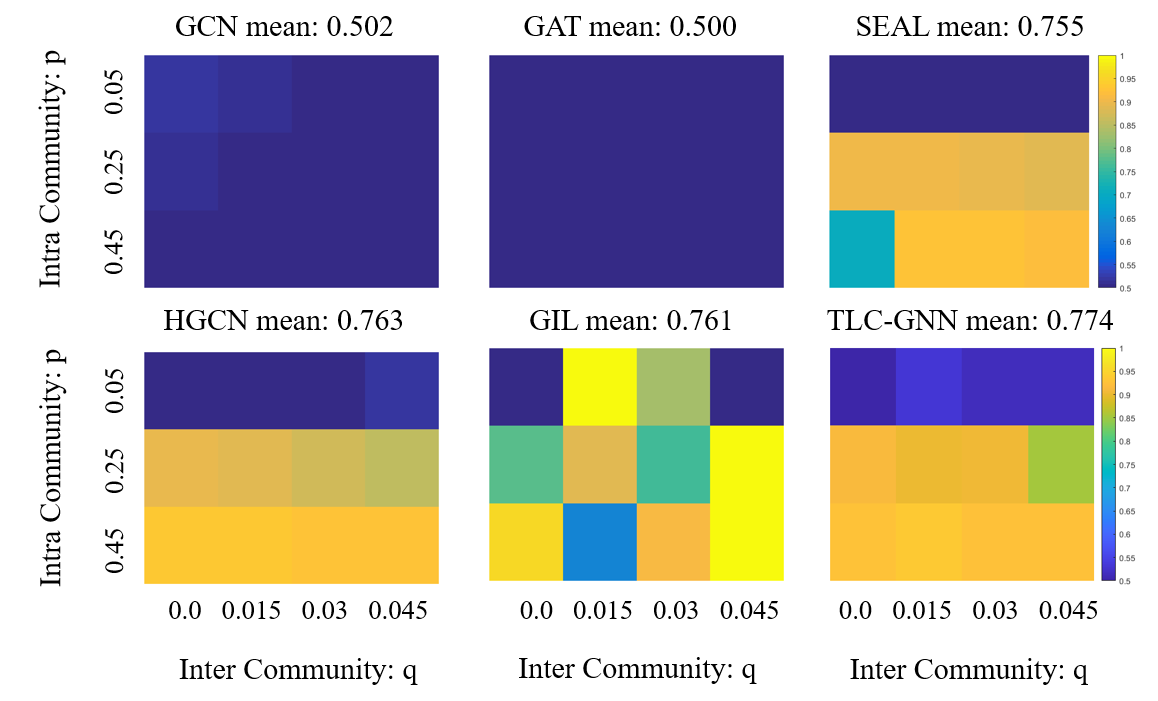}}
		\caption{Heatmap of ROC-AUC score for different methods on synthetic data. For each method we generate a heatmap for its performance on 12 synthetic graphs with different $(p,q)$ combinations. We also report in the title of each heatmap the average performance over 12 graphs.}
		\label{fig:exp_syn}
	\end{center}
\end{figure}

\myparagraph{Results.} Figure~\ref{fig:exp_syn} reports the results of 6 methods. We observe that TLC-GNN outperforms nearly all the SOTA baselines among all the synthetic graphs. This justifies the benefit of pairwise topological feature. We observe that GIL sometimes performs well, but is unstable overall.

\subsection{Real-World Benchmarks}
We use a variety of datasets: (a) PubMed \cite{sen2008collective} is a standard benchmark describing citation network. (b) Photo and Computers \cite{shchur2018pitfalls} are graphs related to Amazon shopping records. (c) PPI networks are protein-protein interaction networks\cite{zitnik2017predicting}. More details of these datasts can be found in the supplementary material.

\begin{table}[bt!]
	\caption{Computational time (seconds per edge) evaluation. }
	
	\label{tab:efficiency}
	\vspace{-.16in}
	\vskip 0.2 in
	\begin{center}
		\begin{small}
			\begin{sc}
				\scalebox{1.0}{
					\begin{tabular}{lccc}
						\hline\noalign{\smallskip}
						& PubMed & Photo & Computers \\
						\noalign{\smallskip}\hline\noalign{\smallskip}
						Alg.~\ref{alg:MR} & 0.0068 & 1.6557 & 4.6531 \\
						Alg.~\ref{alg:Acc} & 0.0027 & 1.1176 & 2.7033  \\
						\noalign{\smallskip}\hline
					\end{tabular}
				}
			\end{sc}
		\end{small}
	\end{center}
	\vskip -0.1in
\end{table}

For pairwise persistence diagrams, we compute intersections of 1-hop neighborhoods for PPI, Photo and Computers due to their high density. We use intersections of 2-hop neighborhoods for PubMed. Following prior works \cite{chami2019hyperbolic,zhu2020graph}, we evaluate link prediction using the ROC-AUC score on the test set. We run the methods on each graph 50 times to get the results.

\begin{table}[t]
	\centering
	\caption{Experimental results(s) on PPI datasets}
	\label{tab:ppi}
	\begin{sc}
		\scalebox{0.8}{
			\begin{tabular}{lccccc}
				\hline\noalign{\smallskip}
				Samples & 1 & 2 & 3 & 4 & 5 \\
				\noalign{\smallskip}\hline\noalign{\smallskip}
				GCN & 75.21 & 74.42 & 77.68 & 76.22 & 69,67\\
				GIL & 57.69 & 1.45 & 34.90 & \textbf{85.61} & 33.65 \\
				HGCN &  &cannot & converge & \\
				SEAL & 50.00 & 64.79  & 67.14 & 72.55 & 50.00\\
				TLC-GNN & \textbf{83.92} & \textbf{81.21}  & \textbf{83.95} & 83.03 & \textbf{83.53} \\
				\noalign{\smallskip}\hline
		\end{tabular}}
	\end{sc}
	\vspace{-.2in}
\end{table}

\myparagraph{Results.} Table~\ref{tab:result} summarizes the performance of all methods. TLC-GNN is superior to others among all the benchmarks. This implies that the high-order topological feature is more effective in these large and dense graphs, which tend to have rich and heterogeneous structures. 

For PPI networks, we run experiments on 5 sample graphs. Ollivier-Ricci curvature is adopted as the filter function. The results are shown in Table~\ref{tab:ppi}. We observe that TLC-GNN consistently outperforms nearly all the SOTA baselines among all the sampled PPI graphs. Although GIL sometimes performs well, it is unstable overall. This further justifies the benefit of pairwise topological feature. 

\textbf{Ablation study.} Comparing with models using nodewise persistent homology (PEGN and TLC-GNN (Nodewise)), TLC-GNN generally performs better. 
This proves that for link prediction task, the proposed pairwise persistent homology carries crucial structural information to model node interaction. Whereas nodewise topological feature cannot.

\subsection{Algorithm Efficiency}
To verify the computational benefit of the proposed new algorithm (Algorithm \ref{alg:Acc}), we compare it with the classic matrix reduction algorithm (Algorithm \ref{alg:MR}) (implemented in the \emph{Dionysus} package \cite{dmitriydionysus}). Both implementations are written in python\footnote{The codes of Dionysus are transformed into Cython, a C-Extension for Python.}. We compare the two algorithms on all real-world benchmarks. We report the average running time for computing the pairwise topological feature for each edge (in seconds). More details are available in the supplementary material.

\myparagraph{Results.} As is shown in Table~\ref{tab:efficiency}, the proposed Algorithm~\ref{alg:Acc} achieves 1.5 to 2.5 times speedup compared with the matrix reduction algorithm (Algorithm~\ref{alg:MR}) on all benchmarks.

\section{Conclusion}
In this paper, we propose a novel pairwise topological feature for link prediction, based on the theory of persistent homology. We also introduce a GNN model to leverage this topological feature. Experiments show that our approach outperforms state-of-the-arts on various benchmarks, especially on large and dense graphs. Besides, we propose a novel algorithm to more efficiently calculate the extended persistence diagrams on graphs. We verify the correctness and efficiency of the algorithm. The algorithm can be generalized to other graph learning tasks.

\bibliography{example_paper}
\bibliographystyle{icml2021}

\twocolumn[





\vskip 0.3in
]




\section*{Appendix}

In this appendix, we provide (1) additional details for persistence image; (2) additional examples to illustrate the matrix reduction algorithm for extended persistent homology; (3) a complete proof of the correctness of the proposed faster algorithm; and (4) additional experimental details and qualitative results.
\section*{A.1 Persistence Image}

In recent years, efforts have been made to map persistence diagrams into representations valuable to machine learning tasks. Persistence Image \cite{adams2017persistence} is one such approach to convert persistence diagrams to vectors, which will be used in our model. Let $T:\mathbf{R}^2 \rightarrow \mathbf{R}^2$ be a linear transformation $T(x,y)=(x,y-x)$ of persistence points. Given a persistence diagram $D$, $T(D) = \{T(d)|d \in D\}$ is the transformed diagram. For any $z \in \mathbf{R}^2$, $\phi_u(z)=\frac{1}{2\pi\sigma^2}e^{-\frac{||z-u||^2}{2\sigma^2}}$ is the 2D Gaussian function with mean $u$ and standard deviation $\sigma$.

Let $\alpha: \mathbf{R}^2 \rightarrow \mathbf{R}$ be a non-negative weight function for the persistence plane $\mathbf{R}^2$. Given a persistence diagram $DgX$, its persistence surface is defined as: $\rho_D(z) = \sum_{u\in T(D)} \alpha(u)\phi_u(z)$. Fix a grid in the plane with $n$ pixels, the persistence image is the collection of pixels $PI_D = \{PI_D[p]\} \in \mathbf{R}^n$ where $PI_D[p] = \int \int_p \rho_{D}(x,y)dxdy$, thus can be directly used in machine learning tasks. The stability of persistence image under perturbation has been proven in \cite{adams2017persistence}. In our setting, $\alpha$ is a piecewise linear weighting function: $$\alpha(x,y)=
\begin{cases}
0 & \text{if } y \leq 0 \\
y & \text{if } 0 < y \leq 1 \\
1 & \text{if } y > 1
\end{cases}.$$

\begin{figure*}[btp]
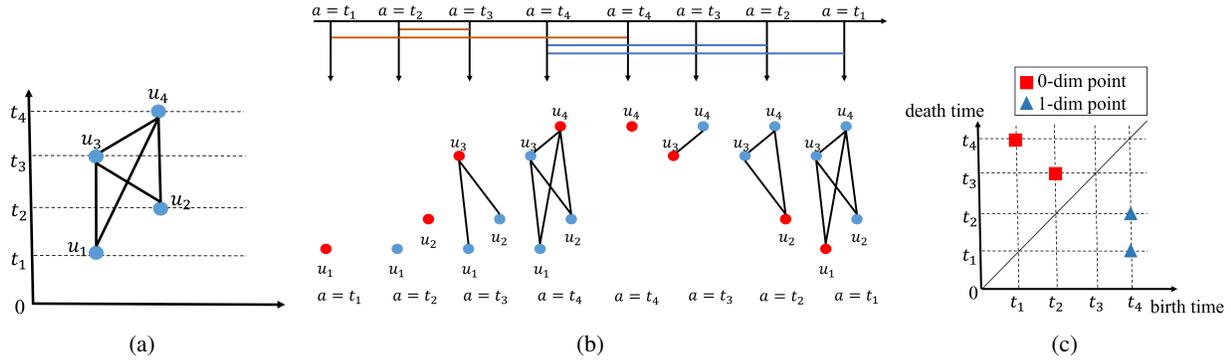

\centering
\subfigure[]{
\begin{minipage}[t]{0.22\linewidth}
\centering
\includegraphics[width=\columnwidth]{PD_ex}
\end{minipage}
}%
\subfigure[]{
\begin{minipage}[t]{0.46\linewidth}
\centering
\includegraphics[width=\columnwidth]{PD}
\end{minipage}%
}%
\subfigure[]{
\begin{minipage}[t]{0.26\linewidth}
\centering
\includegraphics[width=\columnwidth]{PD_diagram}
\end{minipage}
}%
\centering
\vspace{-.1in}
\caption{An illustration of extended persistent homology. (a) We plot the input graph with a given filter function. The filter value for each node is $f(u_1)=t_1$, $f(u_2)=t_2$, $f(u_3)=t_3$, $f(u_4)=t_4$. (b) The ascending and descending filtrations of the input graph. The bars of brown and blue colors correspond to the life spans of connected components and loops respectively. The first four figures are the ascending filtration, while the last four figures denote the descending filtration. In the ascending filtration, $f(uv) = max(f(u),f(v))$, while in the descending filtration, $f(uv) = min(f(u),f(v))$. (c) In the resulting extended persistence diagram, red and blue markers correspond to 0-dimensional and 1-dimensional topological structures. There are two blue markers, corresponding to two loops $(u_1u_3,u_3u_4,u_4u_1)$, $(u_2u_3,u_3u_4,u_4u_2)$. The range of filter function $f$ for these two loops are $[t_1,t_4]$, $[t_2,t_4]$ respectively. These ranges are encoded as the coordinates of the blue markers.}
\label{fig:motivation_appendix}
\vspace{-.15in}
\end{figure*}

\begin{figure*}[btp]
\centering
\subfigure[]{
\begin{minipage}[t]{0.50\linewidth}
\centering
\includegraphics[width=\columnwidth]{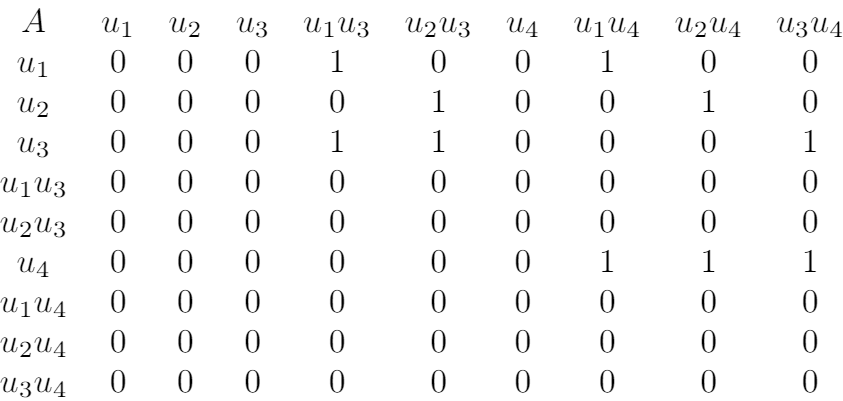}
\end{minipage}
}%
\subfigure[]{
\begin{minipage}[t]{0.50\linewidth}
\centering
\includegraphics[width=\columnwidth]{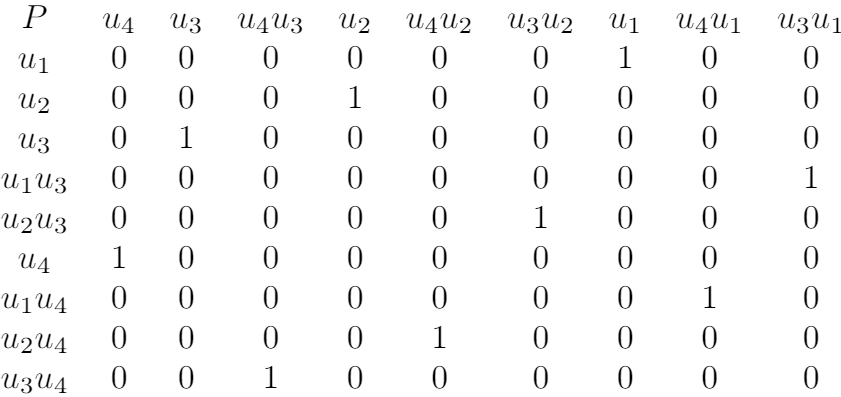}
\end{minipage}%
}%
\\
\subfigure[]{
\begin{minipage}[t]{0.50\linewidth}
\centering
\includegraphics[width=\columnwidth]{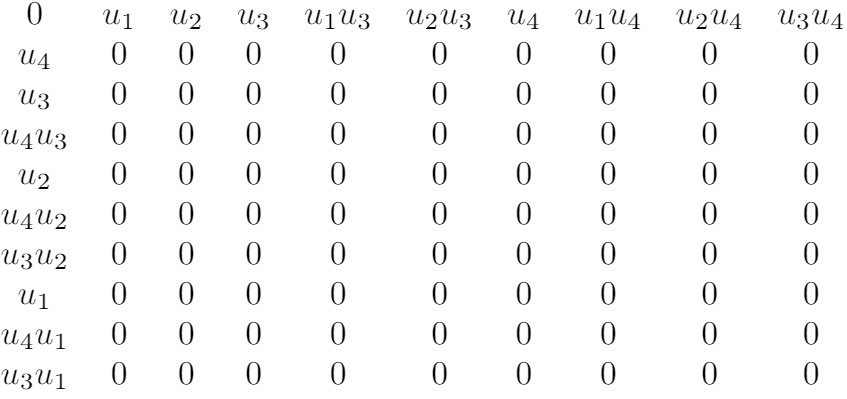}
\end{minipage}
}%
\subfigure[]{
\begin{minipage}[t]{0.50\linewidth}
\centering
\includegraphics[width=\columnwidth]{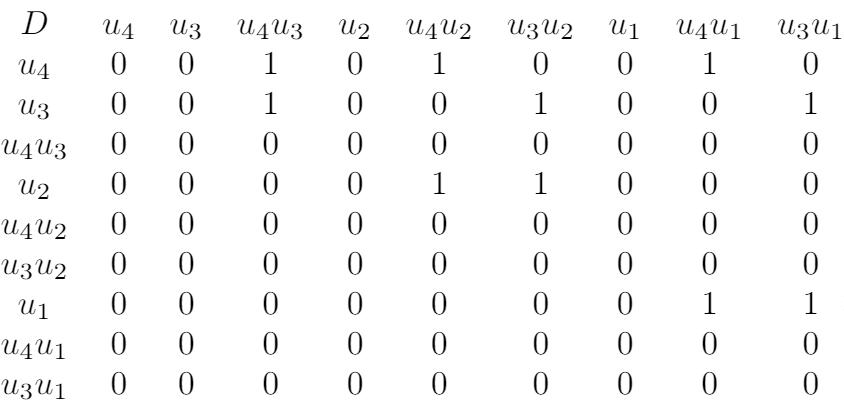}
\end{minipage}
}%
\centering
\vspace{-.1in}
\caption{Reduction matrix $M$ for Figure~\ref{fig:motivation_appendix}}
\label{fig:M}
\vspace{-.15in}
\end{figure*}

\section*{A.2 Examples of the Matrix Reduction Algorithm}
The reduction matrix $M$ for Figure~\ref{fig:motivation_appendix} is shown in Figure~\ref{fig:M}. Recall that $M$ is a binary valued matrix to encode the adjacency relationship between nodes and edges. $M$ is a $2m\times 2m$ matrix consisting of four $m\times m$ matrices: $M=\left[\begin{matrix} A & P \\ 0 & D\end{matrix} \right] $. Every column or row of $M$ corresponds to a simplex. In particular, the first $m$ columns of $M$ correspond to the ascending sequence of simplices $\kappa_1,...,\kappa_m$. The last $m$ columns of $M$ correspond to the descending sequence of simplices $\lambda_1,...,\lambda_m$. The setting is the same for the rows of $M$. Matrix $A$ encodes the relationship between all the simplices in the ascending sequence. Similar to the incidence matrix of a graph, $A[i,j]=1$ iff $\kappa_i$ is the boundary of $\kappa_j$, i.e., $\kappa_i$ is a node adjacent to the edge $\kappa_j$. 

$D$ is defined similarly, except that it encodes the relationship of the simplices in the descending sequence, i.e., $D[i,j] = 1$ iff $\lambda_i$ is the boundary of $\lambda_j$. $P$ stores the permutation that connects the two sequences of simplices, i.e., $P[i,j]=1$ iff $\kappa_i$ and $\lambda_j$ denote the same simplex. $0$ is a zero-valued $m\times m$ matrix. 

Algorithm~\ref{alg:MR} reduces the columns of matrix $M$ from left to right. If we allow certain flexibility in the reduction ordering, we can separate the algorithm into 3 phases: the reduction of matrix $A$ (Phase 1), the reduction of matrix $D$ (Phase 2) and the reduction of matrix $P$ (Phase 3) \cite{cohen2009extending}.
We define a simplex as positive if its corresponding column is zero after reduction, and negative if its corresponding column is not zero after reduction. In our setting, all the nodes, as well as edges that give rise to a loop are positive simplices, edges that destroy a connected component are negative simplices. All the simplices are either negative or positive \cite{edelsbrunner2000topological,edelsbrunner2010computational}. Notice that the positive and negative edges in the ascending filtration are not the same as the positive and negative edges in the descending filtration.

In the following paragraph, we will introduce the whole process of the matrix reduction algorithm by introducing the 3 phases successively.

\subsection*{A.2.1 Phase 1}
Phase 1 is the matrix reduction for $A$. All the columns of nodes and all the rows of edges are all zero in $A$, therefore they will have no impact on the matrix reduction algorithm. After deleting these rows and columns, matrix $A$ is shown below:
\begin{center}
	$\begin{matrix} 
	A &  u_1u_3 & u_2u_3 & u_1u_4 & u_2u_4 & u_3u_4 \\ 
	u_1 & 1 & 0 & 1 & 0 & 0  \\ 
	u_2 & 0 & 1 & 0 & 1 & 0 \\ 
	u_3 & 1 & 1 & 0 & 0 & 1 \\ 
	u_4 & 0 & 0 & 1 & 1 & 1 \\
	\end{matrix} $ \\
\end{center}
For simplicity, we define $low_A(e_i)$ as the maximum row of node $v_j$ for which $A[v_j, e_i] = 1$. Notice that $low_A$ is originally defined for the row index, and we replace the index with the simplex it represents. i.e., we replace $A[1,1]=1$ with $A[u_1,u_1u_3] = 1$, and we replace $low_A(1) = 3$ with $low_A(u_1u_3) = u_3$.

From left to right, we can find that $low_A(u_1u_3) = u_3$, and $low_A(u_2u_3) = u_3 = low_A(u_1u_3)$, thus we add column of $u_1u_3$ to column of $u_2u_3$: $[1,0,1,0] + [0,1,1,0] = [1,1,0,0]$. Notice that ``add" here means the mod-2 sum of the two binary vectors. Thus $low_A(u_2u_3) = u_2$.

Then we can find that $low_A(u_1u_4) = u_4$, $low_A(u_2u_4) = u_4 = low_A(u_1u_4)$, we add the column of $u_1u_4$ to the column of $u_2u_4$: $[1,0,0,1] + [0,1,0,1] = [1,1,0,0]$, thus $low_A(u_2u_4) = u_2 = low_A(u_2u_3)$. Again we add the column of $u_2u_3$ to the column of $u_2u_4$: $[1,1,0,0] + [1,1,0,0] = [0,0,0,0]$. Notice here the column value for $u_2u_3$ is the value after matrix reduction. Therefore $u_2u_4$ is not paired.

Similarly, we add the column of $u_1u_4$ and column of $u_1u_3$ to column of $u_3u_4$ and get $[0,0,0,0]$. Then $u_3u_4$ is not paired. After Phase 1, matrix $A$ is shown below. And we can obtain the persistence pair: $(u_3,u_1u_3)$, $(u_2,u_2u_3)$, $(u_4,u_1u_4)$.
\begin{center}
	$\begin{matrix} 
	A &  u_1u_3 & u_2u_3 & u_1u_4 & u_2u_4 & u_3u_4 \\ 
	u_1 & 1 & 1 & 1 & 0 & 0  \\ 
	u_2 & 0 & 1 & 0 & 0 & 0 \\ 
	u_3 & 1 & 0 & 0 & 0 & 0 \\ 
	u_4 & 0 & 0 & 1 & 0 & 0 \\
	\end{matrix} $ \\
\end{center}

Recall that $A$ encodes the relationship between all the simplices in the ascending sequence. In the ascending sequence, the filter function for an edge $u_iu_j$ is defined as $f_a(u_iu_j) = max(f(u_i),f(u_j))$. Thus we can infer the persistence points from the persistence pairs: $(t_3,t_3)$, $(t_2,t_3)$, $(t_4,t_4)$. We remove the persistence points whose birth time and death time are the same, and get the final persistence point: $(t_2,t_3)$.

Recall that a simplex is defined as positive if its corresponding column is zero after reduction, and negative if its corresponding column is not zero after reduction. So the positive edges in the ascending sequence are $u_2u_4$ and $u_3u_4$, the negative edges in the ascending sequence are $u_1u_3$, $u_2u_3$, and $u_1u_4$. 

\subsection*{A.2.2 Phase 2}
Phase 2 is the matrix reduction for $D$. Notice that we define $low_D$ similar to $low_A$, and Phase 2 influences not only $D$ but also $P$ \cite{cohen2009extending}. All the columns of nodes and all the rows of edges are all zero in $D$. 
In $P$, all the columns of nodes and all the rows of nodes have no impact to the reduction process and pairing for 1-dimensional topology. Therefore for simplicity, we delete these rows and edges and obtain the condensed matrices $P$ and $D$:
\begin{center}
	$\begin{matrix} 
	P & u_4u_3 & u_4u_2 & u_3u_2 & u_4u_1 & u_3u_1  \\ 
	u_1u_3 & 0 & 0 & 0 & 0 & 1 \\ 
	u_2u_3 & 0 & 0 & 1 & 0 & 0\\
	u_1u_4 & 0 & 0 & 0 & 1 & 0 \\
	u_2u_4 & 0 & 1 & 0 & 0 & 0 \\
	u_3u_4 & 1 & 0 & 0 & 0 & 0 \\
	\end{matrix} $ \\
	
	$\begin{matrix} 
	D & u_4u_3 & u_4u_2 & u_3u_2 & u_4u_1 & u_3u_1 \\  
	u_4 & 1 & 1 & 0 & 1 & 0 \\ 
	u_3 & 1 & 0 & 1 & 0 & 1 \\ 
	u_2 & 0 & 1 & 1 & 0 & 0 \\ 
	u_1 & 0 & 0 & 0 & 1 & 1 \\
	\end{matrix} $ \\
\end{center}

From left to right, $low_D(u_4u_3) = u_3$, $low_D(u_4u_2) = u_2$, $low_D(u_3u_2) = u_2 = low_D(u_4u_2)$. We add the column of $u_4u_2$ to the column of $u_3u_2$: $[1,0,1,0] + [0,1,1,0] = [1,1,0,0]$. Then $low_D(u_3u_2) = u_3 = low_D(u_4u_3)$, and we add column of $u_4u_3$ to column of $u_3u_2$: $[1,1,0,0] + [1,1,0,0] = [0,0,0,0]$. Therefore $u_3u_2$ is not paired. Notice that when we add column of $u_4u_2$ and column of $u_4u_3$ to column of $u_3u_2$ in $D$, we are also adding these columns in $P$. Consequently the column of $u_3u_2$ in $P$ will become: $[0,0,0,1,0] + [0,0,0,0,1] + [0,1,0,0,0] = [0,1,0,1,1]$. 

Then we continue the matrix reduction algorithm, $low_D(u_4u_1) = u_1$, and $low_D(u_3u_1) = u_1 = low_D(u_4u_1)$. Then we add column of $u_4u_1$ to column of $u_3u_1$: $[1,0,0,1] + [0,1,0,1] = [1,1,0,0]$, $low_D(u_3u_1) = u_3 = low_D(u_4u_3)$. Again we add column of $u_4u_3$ to column of $u_3u_1$: $[1,1,0,0] + [1,1,0,0] = [0,0,0,0]$. Therefore $u_3u_1$ is not paired, and column of $u_3u_1$ in $P$ becomes $[1,0,1,0,1]$. Matrix $D$ and $P$ will therefore become:

\begin{center}
	$\begin{matrix} 
	P & u_4u_3 & u_4u_2 & u_3u_2 & u_4u_1 & u_3u_1  \\ 
	u_1u_3 & 0 & 0 & 0 & 0 & 1 \\ 
	u_2u_3 & 0 & 0 & 1 & 0 & 0\\
	u_1u_4 & 0 & 0 & 0 & 1 & 1 \\
	u_2u_4 & 0 & 1 & 1 & 0 & 0 \\
	u_3u_4 & 1 & 0 & 1 & 0 & 1 \\
	\end{matrix} $ \\
	
	$\begin{matrix} 
	D & u_4u_3 & u_4u_2 & u_3u_2 & u_4u_1 & u_3u_1 \\  
	u_4 & 1 & 1 & 0 & 1 & 0 \\ 
	u_3 & 1 & 0 & 0 & 0 & 0 \\ 
	u_2 & 0 & 1 & 0 & 0 & 0 \\ 
	u_1 & 0 & 0 & 0 & 1 & 0 \\
	\end{matrix} $ \\
\end{center}

Similar to the process to discover the persistence point, negative edges and positive edges in Phase 1. We can find that all the persistence pairs appear and disappear at the same time, thus there is no persistence point. The positive edges in the descending sequence are $u_3u_2$ and $u_3u_1$. The negative edges in the descending sequence are $u_4u_3$, $u_4u_2$, and $u_4u_1$.

Notice that before Phase 3, all the persistence points are all 0-dimensional persistence points. In other words, they just record the birth and death of connected components. 1-dimensional extended persistence pair will be discussed in Phase 3, where positive edges in the descending filtration will each be paired with an edge in the ascending filtration, thus the saliency of loops can be measured.

\subsection*{A.2.3 Phase 3}
Phase 3 is the reduction process of $P$. Only positive descending edges (edges whose column in $D$ is 0) will be reduced, thus the value in $D$ will not be influenced. Similar to the definition of $low_A$ in Phase 1, we define $low_P(e_i)$ as the maximum row of edge $e_j$ for which $P[e_j, e_i] = 1$. 

From left to right, we only consider the positive descending edges. We find that $low_P(u_3u_2) = u_3u_4$, $low_P(u_3u_1) = u_3u_4 = low_P(u_3u_2)$. We add column of $u_3u_2$ to column of $u_3u_1$: $[0,1,0,1,1] + [1,0,1,0,1] = [1,1,1,1,0]$. 

As a consequence, we can get matrix $P$ and $D$ after Phase3:
\begin{center}
	$\begin{matrix} 
	P & u_4u_3 & u_4u_2 & u_3u_2 & u_4u_1 & u_3u_1  \\ 
	u_1u_3 & 0 & 0 & 0 & 0 & 1 \\ 
	u_2u_3 & 0 & 0 & 1 & 0 & 1 \\
	u_1u_4 & 0 & 0 & 0 & 1 & 1 \\
	u_2u_4 & 0 & 1 & 1 & 0 & 1 \\
	u_3u_4 & 1 & 0 & 1 & 0 & 0 \\
	\end{matrix} $ \\
	
	$\begin{matrix} 
	D & u_4u_3 & u_4u_2 & u_3u_2 & u_4u_1 & u_3u_1 \\  
	u_4 & 1 & 1 & 0 & 1 & 0 \\ 
	u_3 & 1 & 0 & 0 & 0 & 0 \\ 
	u_2 & 0 & 1 & 0 & 0 & 0 \\ 
	u_1 & 0 & 0 & 0 & 1 & 0 \\
	\end{matrix} $ \\
\end{center}

And the extended persistence pair is $(u_3u_4,u_3u_2)$, $(u_2u_4,u_3u_1)$. Notice that  for a extended persistence pair, the latter edge is the positive edge in the descending filtration, and the former edge is its paired edge in the ascending filtration. i.e, in the extended persistence pair $(u_3u_4,u_3u_2)$, $u_3u_2$ is the positive edge in the descending filtration, while $u_3u_4$ is the paired edge in the ascending filtration. Recall that for a certain edge, its filter value in the ascending sequence is defined as the maximum value of the filter value of its nodes: $f_a(u_3u_4) = max(f(u_3),f(u_4)) = max(t_3,t_4) = t_4$. Similarly, $f_a(u_2u_4) = t_4$. And for a certain edge in the descending filtration, its filter value is defined as the minimum value of the filter value of its nodes: $f_d(u_3u_2) = min(f(u_3),f(u_2)) = min(t_3,t_2) = t_2$. Similarly, $f_d(u_3u_1) = t_1$. As a consequence, the extended persistence point are $(t_4,t_2)$ and $(t_4,t_1)$ respectively.

After the whole matrix reduction algorithm, we can get the ordinary persistence diagram for 0-dimensional topological structures (connected components): $[(t_2,t_3)]$ and the extended persistence diagram for 1-dimensional topological structures (loops): $[(t_4,t_2),(t_4,t_1)]$. To get the 0-dimensional extended persistence diagram, the birth and death of the whole connected component is also recorded, that is, the minimum and the maximum filter value $(t_1,t_4)$. So the 0-dimensional extended persistence diagram is $[(t_2,t_3),(t_1,t_4)]$, as shown in Figure~\ref{fig:motivation_appendix} (c).

\section*{A.3 Correctness of the Faster Algorithm}
\label{sec:accal_appendix}
In this section, we provide complete proof of the correctness of the faster algorithm. For convenience, we restate the matrix reduction algorithms Alg.~\ref{alg:MR_appendix} and the proposed faster algorithm Alg.~\ref{alg:Acc_appendix}. We also restate the main theorem \ref{theorem_appendix}.

\begin{theorem}
	\label{theorem_appendix}
	Algorithm~\ref{alg:Acc_appendix} outputs the same extended persistence diagram as Algorithm~\ref{alg:MR_appendix}.
\end{theorem}

To prove Theorem~\ref{theorem_appendix}, the core of our proof is to show that for each positive descending edge, its corresponding pair found by the new algorithm is equivalent to the pair resulting from the matrix reduction algorithm. We prove this by induction. As we go through all positive descending edges in the descending filtration, we show that for each positive edge, which creates a new loop, its extended persistence pair from our algorithm is the same as the reduction result.

In Lemma~\ref{lemma1}, we prove that after reducing the $D$ part of the column of a given edge $e_i$, i.e., when $low_M(e_i)\leq m$, the remaining entries of $e_i$ in $P$ constitute a unique loop in $\{e_i\}\cup \T$.


\begin{minipage}{0.46\textwidth}
\begin{algorithm}[H]
    \centering
	\caption{Matrix Reduction}
	\label{alg:MR_appendix}
	\begin{algorithmic}[1]
		\STATE {\bfseries Input:} filter funtion $f$, graph $G$
		\STATE Persistence Diagram $PD=\{\}$
		\STATE$M=$ build reduction matrix$(f,G)$
		\FOR{$j=1$ {\bfseries to} $2m$}
		\WHILE{$\exists k < j$ with $low_M(k)=low_M(j)$}
		\STATE add column $k$ to column $j$
		\ENDWHILE
		\STATE add $(f(low_M(j)),f(j))$ to $PD$
		\ENDFOR
		\STATE {\bfseries Output:} Persistence Diagram $PD$
	\end{algorithmic}
\end{algorithm}
\end{minipage}
\hfill
\begin{minipage}{0.46\textwidth}
\begin{algorithm}[H]
    \centering
    \caption{A Faster Algorithm for Extended Persistence Diagram}
	\label{alg:Acc_appendix}
	\begin{algorithmic}
	    \STATE {\bfseries Input:} filter funtion $f$ of descending filtration, graph $G$, filter function $f_a$ of ascending filtration
		\STATE 0-dim PD, $E_{pos}$, $E_{neg}$ = Union-Find($G,f$)
		\STATE Tree $\T$ = $E_{neg}$ + all the nodes
		\STATE 1-dim PD = $\{\}$
		\FOR{$e_j$ in $E_{pos}$}
		\STATE assume $e_j=uv$, $Path_u\subseteq \T$ is the path from $u$ to $r$ within the tree $\T$, $Path_v \subseteq \T$ is the path from $v$ to $r$
		\STATE $Loop = Path_u \cup Path_v -  Path_u \cap Path_v$ + $e_j$
		\STATE add $(max_{e\in Loop} f_a(e),f(e_j))$ to 1-dim PD
		\STATE $ e_k = argmax_{e\in Loop} f_a(e)$
		\STATE $\T = \T - \{e_k\} + \{e_j\}$
		\ENDFOR
		\STATE {\bfseries Output:} 0-dim PD, 1-dim PD 
	\end{algorithmic}
\end{algorithm}
\end{minipage}

In Lemma~\ref{lemma2} and Lemma~\ref{lemma3}, we prove that: (1) The lowest entry of the reduced column, $e_j$, is not paired by any other previous edges, and thus will be paired with $e_i$. Indeed, it is the last edge in the loop w.r.t.~the ascending ordering. (2) The updating of the tree $\T = \T - \{e_j\} + \{e_i\}$ is equivalent to adding the reduced column of $e_i$ to all descending columns with nonzero entry at row of $e_j$. Although this will affect the final reduced matrix, it will not change the resulting simplex pairings.

In Lemma~\ref{lemma11}, we further prove inductively that in Phase 3, the remaining entries of $e_i$ in $P$ constitute a unique loop in $\{e_i\}\cup \T_{i-1}$. Here $\T_i$ is the tree after updating the first $i$ positive edges.

Finally, in Lemma~\ref{lemma4}, we prove that the highest filter value and the lowest value exactly form the persistence point of the loop.

\begin{lemma}
	\label{lemma1}
	After Phase 2 (the reduction of descending matrix $D$) and before Phase 3 (the reduction of permutation matrix $P$), for a positive edge $\lambda_j$, the set $\{\kappa_i | P[i,j] = 1\}$ stores the loop that $\lambda_j$ and some of the former negative edges form. Besides, it is the loop that $\lambda_j$ gives birth to.
\end{lemma}

\begin{proof}
    Denote the reduced matrix and its submatrices after Algorithm~\ref{alg:MR_appendix} by $\redM=\left[\begin{matrix} \redA & \redP \\ 0 & \redD\end{matrix} \right] $. As is shown in \cite{edelsbrunner2010computational}, matrix reduction algorithm can be interpreted as computing the reduced matrix i.e., $\redA = AV_1$, $\redD = DV_2$, $\redP = PV_2V_3$, where $V_1$, $V_2$ and $V_3$ are invertible and upper-triangular matrices.
    
    For a positive edge $\lambda_j$, we observe that after Phase 2, its column in $D$ is set to zero. Here $\redD[:,j]$ is denoted as the $j$-th column of matrix $\redD$, and we have $DV_2[:,j] = \redD[:,j] = 0$. According to the definition of loop, the boundary of $\lambda_j$ is finally reduced to zero, so the set $\{\lambda_i|V_2[i,j] = 1\}$ contains all the edges that form the loop which $\lambda_j$ gives birth to. Considering that (1) $V_2$ is upper-triangular, thus only columns of former edges will be added to the column of $\lambda_j$. (2) All the columns of former positive edges in $D$ have been reduced to zero, thus $\lambda_j$ will not be reduced by positive edges. $\{\lambda_i|V_2[i,j] = 1\}$ contains the loop that $\lambda_j$ and some of the former negative edges form. 
    
    In fact, after Phase 2, $P$ represents the matrix $PV_2$. Recall that $P$ stores the permutation that connects the two sequences of simplices, i.e., $P[i; j] = 1$ iff $\kappa_i$ and $\lambda_j$ denote the same simplex. Thus $\{\kappa_i | PV_2[i,j] = 1\}$ stores the same component as $\{\lambda_i|V_2[i,j] = 1\}$, that is the loop $\lambda_j$ and some of the negative edges form. It is the loop that $\lambda_j$ gives rise to.
\end{proof}

In Algorithm~\ref{alg:Acc_appendix}, we first add all the negative edges and all the nodes to form the original tree $\T$. If we add a positive edge $\lambda_j$ to $\T$, then a loop will appear. Considering that in the graph $\T$ + $\{\lambda_j\}$, there is only one loop that $\lambda_j$ gives birth to, and it consists of $\lambda_j$ and some of the negative edges that appear before $\lambda_j$ in the descending sequence. Therefore, it is exactly the loop that set $\{\kappa_i | P[i,j] = 1\}$ consists of after Phase 2. 

In the following lemmas, we try to show that: In Algorithm~\ref{alg:Acc_appendix}, replacing the paired edge $e_k$ with the positive edge $e_j$ every step has the same result with the same step in Phase 3 (matrix reduction for $P$).

\begin{lemma}
	\label{lemma2}
	In Algorithm~\ref{alg:Acc_appendix}, the process of replacing the positive edge $e_j$ with its paired edge $e_k$ is equivalent to adding the column of $e_j$ to all the columns whose row of $e_k$ is one.
\end{lemma}

\begin{proof}
    If we add the column of $e_j$ to all the later columns whose row of $e_k$ is ``one", all the rows of $e_k$ in the later columns will be zero. i.e., assume $P[e_k,e_l] = 1$ \footnote{Here, for simplicity, we use $e_k$ and $e_l$ to represent their indices.}, then after adding column of $e_j$ with column of $e_l$, $P[e_k,e_l] = 0$. The set $\{e_i|P[e_i,e_l] = 1\}$ contains the loop that appear before $e_l$ (without $e_k$) and $e_j$, $e_l$.
    
    Here we denote $\T_{e_j} = \T - \{e_k\} + \{e_j\}$. In the graph $\T_{e_j}$ + $\{e_l\}$, there is only a loop. Therefore, it is exactly the loop that set $\{e_i|P[e_i,e_l]\}$ consists of.
    
    For the former columns whose row of $e_k$ is ``one", i.e., $P[e_k,e_l] = 1$. This means $low_P(e_l) > e_k$ \footnote{If $low_P(e_l) = e_k$, $e_k$ will be paired in the former columns, thus dissatisfies the assumption that $low_P(e_j) = e_k$.}, then adding column $e_j$ to $e_l$ will not change the extended persistence pair because $low_P(e_l)$ will not be affected by previous edges and thus remain the same.
\end{proof}

\begin{lemma}
	\label{lemma3}
	Adding the column of $e_j$ to all the columns whose row of $e_k$ is one in Lemma~\ref{lemma2} has the same extended pair as the matrix reduction algorithm in Phase 3.
\end{lemma}

\begin{proof}
    Assume $P[e_k,e_l] = 1$, for simplicity, we define $low_P(e_l)$ as the maximum row index $e_i$ for which $P[e_i, e_l] = 1$. If $low_P(e_l) = e_k$, then in the matrix reduction algorithm, we should add column $e_j$ to $e_l$, thus the two algorithms are exactly the same. If $low_P(e_l) \neq e_k$, which means $low_P(e_l) > e_k$, then adding column $e_j$ to $e_l$ will not change the extended persistence pair because $low_P(e_l)$ will not be affected by previous edges and thus remain the same. As a consequence, adding the column of $e_j$ to all the later columns whose row of $e_k$ is one in Lemma~\ref{lemma2} has the same extended pair with the matrix reduction algorithm in Phase 3. Together with Lemma~\ref{lemma2}, adding the column of $e_j$ to all the columns whose row of $e_k$ is one in Lemma~\ref{lemma2} has the same extended pair with the matrix reduction algorithm in Phase 3.
\end{proof}

From Lemma~\ref{lemma2} and Lemma~\ref{lemma3}, we manage to prove that in Algorithm~\ref{alg:Acc_appendix}, replacing the paired edge $e_k$ with the corresponding positive edge $e_j$ is equivalent to Phase 3 (matrix reduction for $P$). Combining it with Lemma~\ref{lemma1}, we can prove that for every positive edge, to update $\T$ in the faster algorithm leads to the same extended pair with the matrix reduction algorithm. We have proved that in a single step, the two algorithms are equivalent. Then we should prove inductively that the whole process of the Algorithm~\ref{alg:Acc} is equivalent to the matrix reduction Algorithm~\ref{alg:MR_appendix}. 

\begin{lemma}
	\label{lemma11}
	In Algorithm~\ref{alg:Acc}, the process to update the tree $\T$ is equivalent to the matrix reduction process in Phase 3.
\end{lemma}

\begin{proof}
	First, in the original lemmas, we have proved that adding a positive edge to the original tree $\T$ and updating the tree leads to the same result as the matrix reduction algorithm.
	
	We then assume that after adding the first $j-1$ positive edges, the process of updating the tree can output the same results as the matrix reduction algorithm. Denote the tree after updating the first $j-1$ positive edges as $\T_{j-1}$. 
	
	When adding the $j$-th positive edge $e_j$. Similar to the prove in Lemma~\ref{lemma1}, we can prove that the set $\{e_i|P[e_i,e_j] = 1\}$ stores the loop that $e_j$ and $\T_{j-1}$ form. And similar to the prove in Lemma~\ref{lemma2} and Lemma~\ref{lemma3}, we can prove that replacing the paired edge $e_k$ with the newly added positive edge $e_j$ leads to the same extended persistence pair with the matrix reduction algorithm in Phase 3. As a consequence, the process of updating the tree to $\T_j$ can lead to the same results as the reduction algorithm. Then Lemma~\ref{lemma11} is proved inductively. 
\end{proof}

In the above Lemmas, we have proven that to update the tree $\T$ in the faster algorithm output the same result as the matrix reduction algorithm. Then we should confirm that the extended point given by Algorithm~\ref{alg:Acc_appendix} is correct.

\begin{lemma}
	\label{lemma4} 
	In a loop, the highest filter value and the lowest filter value form its extended persistence point.
\end{lemma}

\begin{proof}
	For a positive edge $e_j$, the edge it pairs in matrix $P$ is the lowest one in its column, representing the latest one in the ascending filtration. Notice that in the ascending filtration, we define the filter value of an edge $f_a(uv) = max(f(u),f(v))$, thus it has the biggest filter value in the loop. $e_j$ is the latest born edge in the descending filtration, Recall that we define the filter value of an edge in the descending filtration as $f_d(uv) = min(f(u),f(v))$, thus it contains the lowest filter value in the loop. As a result, the extended persistence point will be the highest and lowest filter value of the loop. 
\end{proof}

From Lemma~\ref{lemma4}, we manage to prove that the persistence point of the positive edge $e_j$ is exact the value provided in Algorithm~\ref{alg:Acc_appendix}. As a consequence, we can justify Theorem~\ref{theorem_appendix}: Algorithm~\ref{alg:Acc_appendix} outputs the same extended persistence diagram as Algorithm~\ref{alg:MR_appendix}.

\section*{A.4 Experiment}

\subsection*{A.4.1 Introduction to Real-World Datasets}
The real-world datasets in this paper include:
\begin{enumerate}
	\item Citation network: PubMed \cite{sen2008collective} is a standard benchmark describing citation network where nodes denote scientific papers and edges are citations between them.
	\item Amazon networks: Photo and Computers \cite{shchur2018pitfalls} are datasets related to Amazon shopping records where nodes represent products and edges imply that two products are frequently brought together.
	\item PPI networks: 24 Protein-protein interaction networks\cite{zitnik2017predicting} where nodes denote protein and edges represent the interaction between proteins. Each graph has 3000 nodes with average degree 28.8. The dimension of node feature vector is 50.
\end{enumerate}
The detailed statistics of these data is shown in Table~\ref{tab:data}. Because PPI networks contain multiple graphs, we do not add it in the Table.

	\begin{table}[t]
		\caption{Statistics of benchmark datasets}
		\label{tab:data}
		\begin{center}
			\begin{small}
				\begin{sc}
					\scalebox{0.9}{
						\begin{tabular}{lcccc}
							\hline\noalign{\smallskip}
							Dataset & Features & Nodes & Edges & Edge density \\
							\noalign{\smallskip}\hline\noalign{\smallskip}
							Pubmed  & 500 &  19717 & 44338 &  0.0002  \\
							Photo & 745 & 7487 & 119043 & 0.0042 \\
							Computers & 767 & 13381 & 245779 & 0.0027 \\
							
							\noalign{\smallskip}\hline
						\end{tabular}
					}
				\end{sc}
			\end{small}
		\end{center}
		\vskip -0.1in
	\end{table}

\begin{figure*}[btp]{
		\centering
		\begin{center}
			\includegraphics[width=\columnwidth]{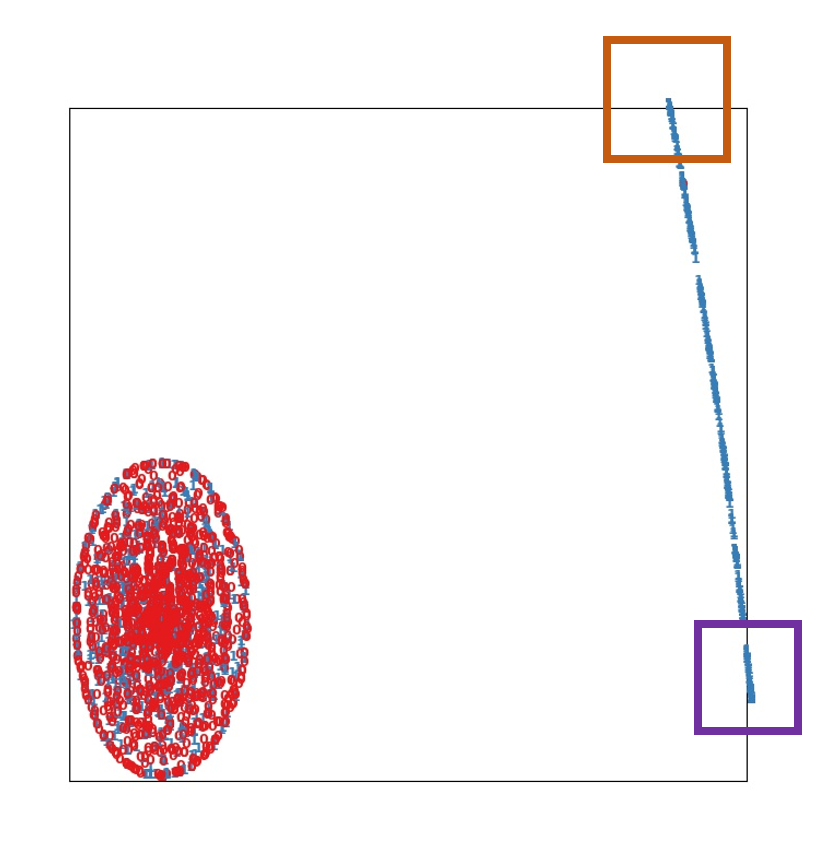}
		\end{center}
	}%

\centering
\vspace{-.1in}
\caption{The t-distributed stochastic neighbor embedding (t-SNE) projection of the persistence images on PubMed. In the figure, the red and blue marks denote the persistence images generated by the negative edge and positive edges respectively.}
\label{fig:tsne}
\vspace{-.15in}
\end{figure*}

\begin{figure*}[btp]
\centering
\subfigure[]{
\begin{minipage}[t]{0.50\linewidth}
\centering
\includegraphics[width=.5\columnwidth]{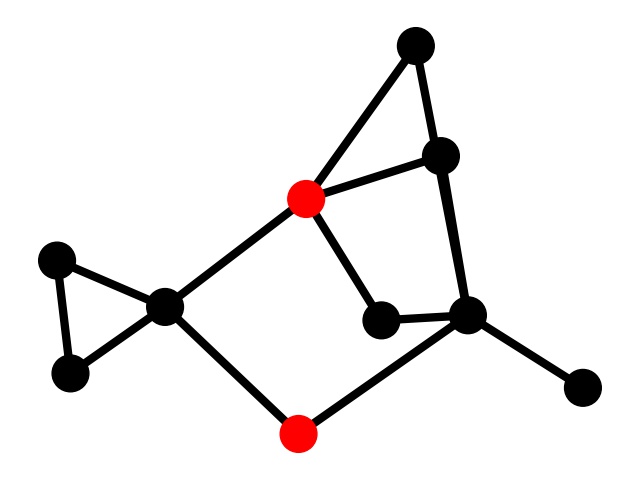}
\end{minipage}
}%
\subfigure[]{
\begin{minipage}[t]{0.50\linewidth}
\centering
\includegraphics[width=.5\columnwidth]{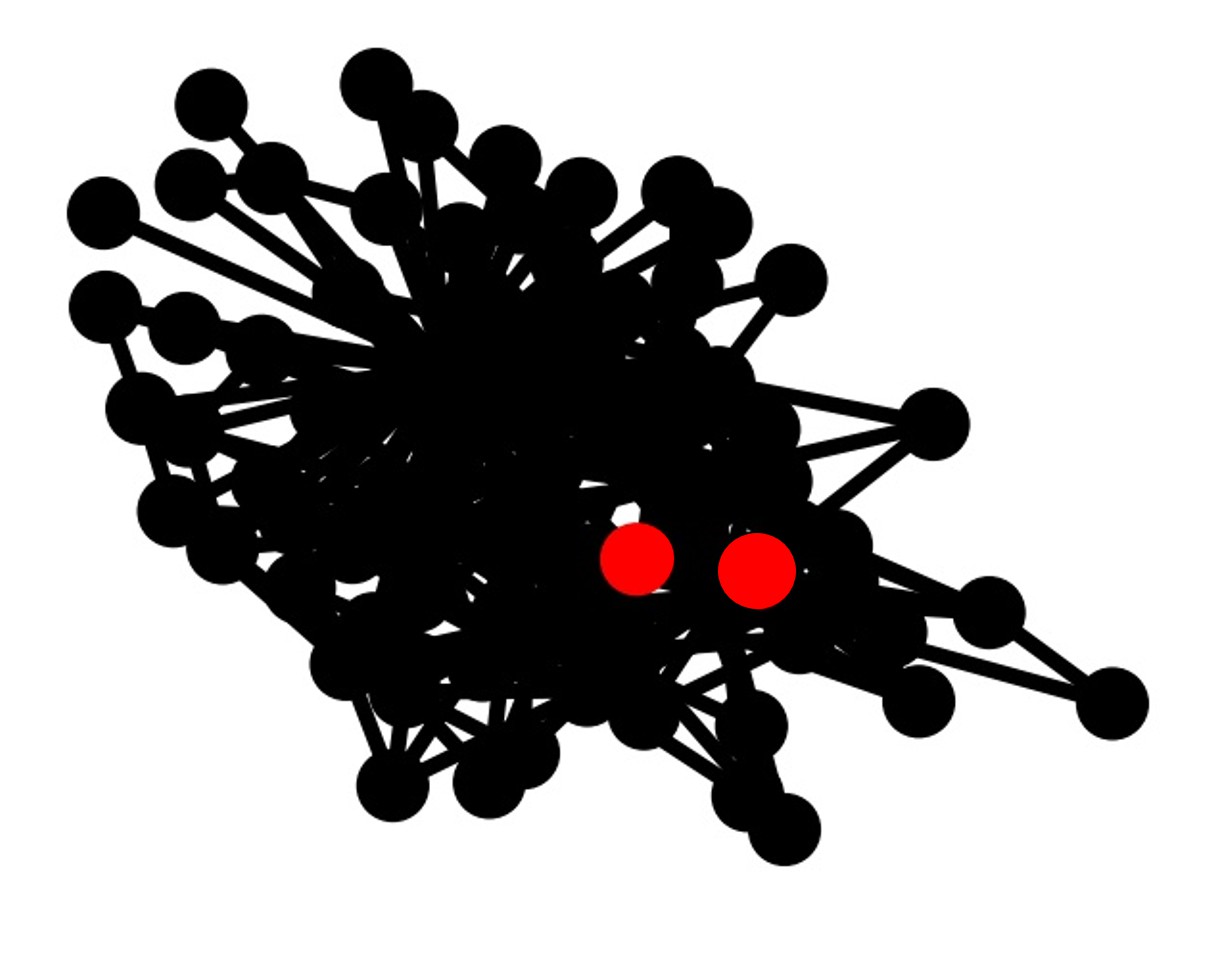}
\end{minipage}
}%
\\
\subfigure[]{
\begin{minipage}[t]{0.50\linewidth}
\centering
\includegraphics[width=0.8\columnwidth]{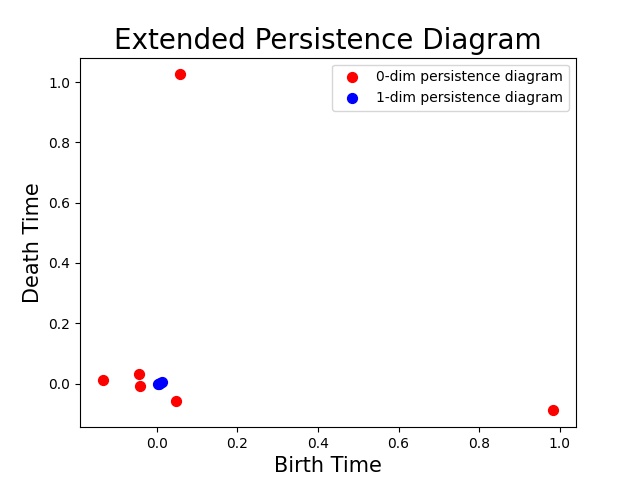}
\end{minipage}%
}%
\subfigure[]{
\begin{minipage}[t]{0.50\linewidth}
\centering
\includegraphics[width=0.8\columnwidth]{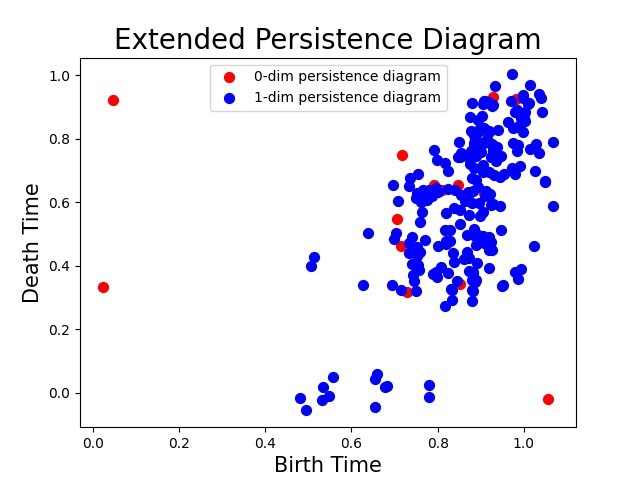}
\end{minipage}%
}%
\centering
\vspace{-.1in}
\caption{We sample two positive edges from the visualization of PubMed in Figure~\ref{fig:tsne}. We draw their subgraphs and diagrams. (a) and (c) represent the sample in the brown box. (b) and (d) represent the sample in the purple box. 
}
\label{fig:pubmed_tsne}
\vspace{-.15in}
\end{figure*}

\subsection*{A.4.2 Detailed Experiment Settings}
\myparagraph{Data split.}  We follow the experimental setting from \cite{chami2019hyperbolic,zhu2020graph} and use 5\% (resp.~10\%) of existing links as positive samples of the validation set (resp.~test set). And equal number of non-existent links are sampled as negative samples of the validation and test set. 
The remaining 85\% existing links are used as the positive training set. In every epoch, we randomly choose the same number of remaining non-existent links as the negative training set. We report the results on test set when the models achieve the best performance on the validation set. 

\myparagraph{Training setting.} On synthetic experiments, we run all the methods on each graph 10 times and report the mean average area under the ROC curve (ROCAUC) scores as the result. On real-word benchmarks, we run all the methods on each graph 50 times and report the mean and standard deviation of ROCAUC scores as the result. All methods use the following training strategy: the same training epochs (2000), and the same early stopping on validation set with 200 patience epochs. The only exception is SEAL; due to its slow training speed and fast convergence, we only train 200 epochs. 

Following \cite{chami2019hyperbolic,zhu2020graph}, during training, we remove positive validation and test edges from the graph.  Cross Entropy Loss is chosen as the loss function and Adam is adopted as the optimizer with the learning rate set to 0.01 and weight-decay set to 0. For fairness, we set the number of node embeddings of the hidden layer and the final layer to be the same (100 and 16) for all networks.  The backbone GNN in our model is a classic 2-layer GCN with one hidden and one output layer. All persistence images in the experiments are 25-dimension. All the activation function used in the graph neural networks is RELU and all the activation function used in Fermi-Dirac decoder (needed in TLC-GNN) is Leakyrelu with negative slope set to 0.2.  

\myparagraph{Details on evaluation of algorithm efficiency.} To evaluate the efficiency of the proposed faster algorithm (Section 5.3 in the main paper), we use the following setting. For the sparse graph PubMed, we compute 0-dimensional and 1-dimensional extended persistence diagrams on all the existing edges. No edges are removed. For large and dense graphs like Photo and Computer, we compute 0-dimensional and 1-dimensional extended persistence diagrams on the first 1000 edges in the default edge list. We run the algorithms on each graph 10 times and report the average seconds per edge as the result. We use a cluster with two Intel Xeon Gold 5128 processors and 192GB RAM to run the two algorithms without multi-threading.

\subsection*{A.4.3 Further experiments}
In this paragraph, we add experiments to evaluate the effect of $k$ (the hop distance to form the enclosing subgraph).
Considering that on large and dense graphs such as Photo and Computers, it costs immensely to compute the persistence image when $k$ is 2, and on all the datasets, computing persistence image when $k$ is larger than 3 takes immense computational cost, we evaluate the effect of $k$ on PubMed and a sampled graph in PPI. As shown in Table~\ref{tab:k}, $k=2$ is generally a good choice in these two datasets. However, it cost much more to compute the persistence images in PPI networks when $k$ = 2, so we finally set $k$ = 1 in PPI datasets.
\begin{table}[t]
	\centering
	\caption{Experimental results(s) on the chosen of hop distance}
	\label{tab:k} 
	\scalebox{0.9}{
		\begin{tabular}{lcc}
			\hline\noalign{\smallskip}
			$k$ & 1 & 2 \\
			\noalign{\smallskip}\hline\noalign{\smallskip}
			PubMed & 96.79 & \textbf{97.03}\\ 
			PPI & 83.92 & \textbf{84.11} \\
			\noalign{\smallskip}\hline
	\end{tabular}}
\end{table}

\subsection*{A.4.4 Visualization of Extended Persistence}
We provide qualitative examples to further illustrate how topological features can help differentiate edges/non-edges.

We use the t-distributed stochastic neighbor embedding (t-SNE) \cite{van2008visualizing} to map the 25-dim persistence images of samples to a 2D plane, as shown in Figure~\ref{fig:tsne}. The persistence images here are created using the Ollivier-Ricci curvature \cite{ni2018network} as the filter function. For each graph, we randomly choose 1500 positive edges and 1500 negative edges.
Figure~\ref{fig:tsne} shows the t-SNE results of PubMed. The red and blue marks represent negative and positive edges respectively. 
Despite some exceptions, negative edges and positive edges are well separated in terms of persistent homology features.  

To further understand the data, we choose 2 positive samples from the t-SNE plot of PubMed (Figure \ref{fig:tsne}) and draw their local enclosing graphs and persistence diagrams. In the t-SNE plot, positive samples form an elongated linear structure. We intentionally sample the two samples from the two ends of the structure. One from the center of the brown box. The other from the center of the purple box. The local enclosing graph and diagram of the first sample is drawn in Figure \ref{fig:pubmed_tsne} (a) (c). The graph and diagram of the second sample is drawn in Figure \ref{fig:pubmed_tsne} (b) (d).


In the graph, the red nodes denote the target nodes, while the other nodes are black. In the persistence diagram, the red and blue markers represent the 0-dimensional and 1-dimensional persistence points respectively. Notice that we add a random jitter to each persistence point so that we can observe the overlapped persistence points.

For the sample in the brown square, we observe that in its enclosing subgraph, there are several loops passing the target nodes. They correspond to 1D persistence points with death time zero in the diagram. The only loop that does not pass the target nodes has the same birth and death time, thus is not shown in the persistence diagram. 

For the sample in the purple square, we observe that there exist many loops in the generated subgraph, and the distribution of the 1-dimensional persistence points mainly concentrate on the top right of the diagram. In addition, more 0-dimensional extended persistence points whose birth time is smaller than its death time appear. 
We observe (1) the density of 1-dimensional extended persistence points gradually increase from the bottom to the upper-right of the diagram. (2) more 0-dimensional extended persistence points from the ascending filtration appear. 

\myparagraph{Discussion.}
From Figure~\ref{fig:tsne} and Figure~\ref{fig:pubmed_tsne}, we observe the following phenomena. Persistence images effectively differentiate positive and negative edges in all graphs. While almost all negative samples form a tight cluster, positive samples form clusters like pieces of 1-manifolds. This makes us wonder whether these clusters can be parameterized by a latent parameter. The selected two samples further suggest the possibility of this hypothesis. The brown and purple samples represent the two extreme of the positive cluster in PubMed. The share common characteristics, e.g., both have rich loops (compared to their number of nodes). Meanwhile, they range from small subgraphs with less loops to dense subgraphs with many loops. To investigate further on these positive sample clusters is an interesting research direction in the future.


\appendix


\end{document}